\documentclass{article}

\usepackage[a4paper, total={6in, 9in}]{geometry}

\usepackage{bm}
\usepackage{url} 
\usepackage{dsfont}
\usepackage{xcolor}
\usepackage{amsthm}
\usepackage{amsmath}
\usepackage{amssymb}
\usepackage{multicol}
\usepackage{mathtools}

\usepackage{algorithm}
\usepackage{algorithmic}

\usepackage{adjustbox}
\usepackage{graphicx}
\usepackage{booktabs}

\usepackage[round]{natbib}

\usepackage[toc, page, titletoc]{appendix}

\newtheorem{lemma}{Lemma}

\newtheorem{theorem}{Theorem}

\newtheorem{assumption}{Assumption}

\newenvironment{tproof}{\proof}{\endproof}

\usepackage{thm-restate}

\DeclareMathOperator*{\argmin}{arg\,min}
\DeclarePairedDelimiter{\abs}{\lvert}{\rvert}
\DeclarePairedDelimiter{\norm}{\lVert}{\rVert}

\title{Optimism and Delays in Episodic Reinforcement Learning}

\author{Benjamin Howson \and Ciara Pike-Burke \and Sarah Filippi}
\date{
    Department of Mathematics, Imperial College London\\[2ex]
    \today
}

\allowdisplaybreaks
\begin{document}
\maketitle
\begin{abstract}
    There are many algorithms for regret minimisation in episodic reinforcement learning. This problem is well-understood from a theoretical perspective, providing that the sequences of states, actions and rewards associated with each episode are available to the algorithm updating the policy immediately after every interaction with the environment. However, feedback is almost always delayed in practice. In this paper, we study the impact of delayed feedback in episodic reinforcement learning from a theoretical perspective and propose two general-purpose approaches to handling the delays. The first involves updating as soon as new information becomes available, whereas the second waits before using newly observed information to update the policy. For the class of optimistic algorithms and either approach, we show that the regret increases by an additive term involving the number of states, actions, episode length, the expected delay and an algorithm-dependent constant. We empirically investigate the impact of various delay distributions on the regret of optimistic algorithms to validate our theoretical results.  
\end{abstract}

\section{Introduction}
Episodic Reinforcement Learning (RL) considers the problem of an agent learning how to act in an unknown environment to maximise its cumulative reward. The problem formulation is broad enough to capture the nature of sequential decision-making in many real-world scenarios as it permits complex dependencies between actions, rewards and future environmental states. Despite the complexity of the learning problem, there are many provably efficient algorithms for this problem setting \citep{UCRL2, KLUCRL, UCRL2B, UCBVI, UBEV}.\\ 

These existing algorithms focus on the traditional model where one assumes that the algorithm updating the policy observes the sequence of states, actions and rewards at the end of every episode. Unfortunately, this immediate feedback assumption is unrealistic in almost all practical applications. In healthcare, for example, feedback relating to a patient on a particular treatment protocol is not observable to the policy maker until they return to the clinic at a scheduled time point in the future. In e-commerce, one observes a conversion at some unknown time long after a sequence of recommendations. Yet another example is wearable technology. Here, the heavy computation involved in policy updating must occur on a separate machine, forcing the communication of information, which naturally introduces a delay between the agent collecting feedback and the policy updater. In any of these scenarios, the algorithm must continue operating, despite lacking information from its past choices.\\ 

The above examples illustrate that delayed feedback is a fundamental challenge in real world reinforcement learning. Unfortunately, there is little theoretical understanding of the impact of delays in episodic reinforcement learning in the existing literature. We seek to fill this gap in the literature in this paper.

\subsection{Related Work}
Recently, the topic of delays has attracted a lot of attention in the bandit setting \citep{AA2011,  MD2011, PJ2013, Mandel2015, CV2017, Pike-Burke2018, Zhou2019, AM2020, CV2020}. Here, the feedback is the reward associated with the chosen action in each round. Perhaps the most appealing approach in the multi-armed bandit setting is the queuing technique, which shows that the delays cause an additive penalty involving the expected delay for any base algorithm \citep{PJ2013, Mandel2015}. The high-level idea is to build a meta-algorithm that creates a simulated non-delayed environment for any base algorithm designed for immediate feedback, such as UCB1 or KL-UCB. They achieve this by introducing a mechanism that stores the rewards for each action in separate queues and having the base algorithm interact with these rather than the actual environment. Unfortunately, the queuing technique does not readily extend to the delayed feedback setting in RL, as forming the queues would require knowledge of the state and action seen in each step of an episode; this information is delayed in our setting.\\

\cite{PJ2013} present another meta-algorithm for adversarial multi-armed bandits with delayed rewards that is trivial to adapt to our setting. They propose creating a new instance of the chosen base algorithm whenever there is no feedback, allowing one to bound the regret of each instance separately using standard techniques. More precisely, this involves maintaining $\tau_{\max} + 1$ versions of the algorithm, where $\tau \leq \tau_{\max}$ almost surely \citep{PJ2013}. Thus, the regret of taking this approach is multiplicative, as the maximal delay scales the regret of the base algorithm. \\

Previous work in RL has considered constant delays in observing the current state in Markov Decision Processes (MDPs) \citep{KVK2003}. More recent work considers delayed feedback in adversarial MDPs \citep{DAMDP}. They developed an algorithm that computes stochastic policies based on policy optimisation. The regret of this algorithm depends on the sum of the delays, the number of states and the number of steps per episode. For stochastic MDPs, they state a regret bound of the form $H^{3/2}S \sqrt{AT} + H^2 S\tau_{\max}$, where $H$ is the number of decisions the learner must make per episode, $T = KH$ is the total number of decisions made across all $K$ episodes, $S$ is the number of states in the environment, $A$ is the number of actions and $\tau_{k} \leq \tau_{\max}$. However, the leading order term in their regret bound is loose for many base algorithms. Their approach also requires \textit{a-priori} knowledge of the maximal delay to define a phase of explicit exploration; this quantity is often unknown in many practical applications. Further, the base algorithm accrues linear regret in this exploration phase, and the maximal delay can be prohibitively large. We propose two approaches that avoid such prior knowledge and can leverage new information in the early episodes much faster, leading to tighter algorithm-specific theoretical results and better empirical performance. In addition to the improved theoretical results, we relax the assumption that the delay distribution has a finite and known maximum, and instead only require that the delays have a finite expectation that we assume is unknown.

\subsection{Contributions}
The delayed feedback model studied in this paper poses several theoretical challenges that do not arise in the standard episodic reinforcement learning problem, such as delayed updates and disentangling the delays from the difficulty of the learning problem in the theoretical analysis.\\

We introduce two novel meta-algorithms to overcome these challenges, namely \textit{active} and \textit{lazy} updating. Both take any algorithm as input and transform it into an algorithm that can handle delayed feedback. Henceforth, we refer to the input algorithm as the \textit{base algorithm}. Using these meta-algorithms, we obtain high probability regret bounds for any optimistic model-based base algorithm in the delayed feedback setting. For both active and lazy updating, the penalty for delayed feedback is an additive term involving the expected delay. Although they obtain similar theoretical results, active and lazy updating employ different algorithmic ideas to separate the delays from the learning problem in the theoretical analysis. \\

The active updating meta-algorithm uses the base algorithm to update the policy as soon as it observes feedback from the environment. Deriving theoretical guarantees for active updating involves tackling the delays head-on, as the delays force the policy to remain constant across numerous episodes. Consequently, the learner can repeatedly make sub-optimal decisions. To quantify the impact of delayed feedback, we introduce several techniques that carefully separate the difficulty of the learning problem from the delays.  \\

The lazy meta-algorithm works slightly differently. Instead of updating immediately, it waits for the amount of feedback to surpass some threshold before updating the policy. One can control this threshold, and therefore the frequency of policy updates, through a hyperparameter $\alpha$. By waiting to update, lazy creates a simulated non-delayed version of the environment for the input algorithm, allowing us to handle the delays separately from the difficulty of the learning problem. \\

\section{Preliminaries}
We consider the task of learning to act optimally in an unknown episodic finite-horizon Markov Decision Process, EFH-MDP. An EFH-MDP is formalised as a quintuple: $M = \left(\mathcal{S}, \mathcal{A}, H, P, R\right)$. Here, $\mathcal{S}$ is the set of states, $\mathcal{A}$ is the set of actions, $H$ is the horizon and gives the number of steps per episode, $P = \{P_{h}(\cdot \vert s, a)\}_{h, s, a}$ is the set of probability distributions over the next state and $R = \{R_{h}(s, a)\}_{h, s, a}$ is the set of reward functions. For conciseness, we assume that the reward function is known, deterministic and bounded between zero and one for all state-action-step triples.\footnote{The main challenge in model-based reinforcement learning lies in estimating the transition function. Thus, an extension to unknown bounded stochastic rewards is relatively straightforward.}\\

In the episodic reinforcement learning problem, the base algorithm interacts with an MDP in a sequence of episodes: $k = 1, 2, \dots, K$. We denote the set of episodes by: $[K] = \{1, 2, \dots K\}$; a convention that we adopt for sets of integers. In this paper, we consider base algorithms that compute a deterministic policy $\pi_k: \mathcal{S}\times [H]\rightarrow \mathcal{A}$ at the start of each episode $k\in[K]$. It is known that in finite horizon stochastic MDPs, if an optimal policy exists, there is a deterministic optimal policy \citep{MP1994}. \\

Once the base algorithm has computed a policy, an agent uses said policy to sample feedback from the environment by: selecting an action, $a_{h}^{k} = \pi_{k}(s_{h}^{k}, h)$; receiving a reward, $r_{h}^{k} = R_{h}(s_{h}^{k}, a_{h}^{k})$; and transitioning to the next state, $s_{h + 1}^{k} \sim P_{h}(\cdot \vert s_{h}^{k}, a_{h}^{k})$; for each $h = 1, \cdots ,H$. The feedback associated with the $h$-th step of the $k$-th episode is given by: 
\begin{equation}
    \mathcal{D}_{h}^{k} \coloneqq \{(s_{h}^{k}, a_{h}^{k}, r_{h}^{k}, s_{h + 1}^{k})\}\;.
\end{equation}
We measure the quality of a policy, $\pi$, using the value function, which is the expected return at the end of the episode from the current step, given the current state: 
\begin{equation}\label{eqn: value function}
    V^{\pi}_{h}\left(s\right) = \mathbb{E}_{\pi}\left[\,\sum_{h' = h}^{H} r_{h'}^{k} \Big\vert s_{h'}^{k} = s\right]\;.
\end{equation}
Further, we denote the optimal value function by: $V^{*}_{h}(s) = \max_{\pi} \{V^{\pi}_{h}(s)\}$, which gives the maximum expected return over deterministic policies $\forall (s, h)\in\mathcal{S}\times [H]$. When evaluating reinforcement learning algorithms, it is common to use regret: 
\begin{equation}\label{eqn: regret}
    \mathfrak{R}_{K} = \sum_{k = 1}^{K}V^{*}_{1}\left(s_{1}^{k}\right) - V^{\pi_{k}}_{1}\left(s_{1}^{k}\right) \coloneqq \sum_{k = 1}^{K}\Delta_{1}^{k}\;.
\end{equation}
Throughout, $T = KH$ denotes the total number of steps. \citet{TDMDPLB} show that the lower bound for the regret in the standard episodic reinforcement learning setting with stage-dependent transitions is: $\Omega(H\sqrt{SAT})$.

\subsection{Regret Minimisation in Model-Based RL}\label{sec: regret minimisation}
Many provably efficient algorithms exist for learning in EFH-MDPs when feedback is immediate. In this paper, we focus on the large class of optimistic model-based reinforcement learning algorithms. These algorithms maintain estimators of the transition probabilities for each $(s, a, s')\in\mathcal{S}\times\mathcal{A}\times\mathcal{S}$:
\begin{align*}
    \hat{P}_{kh}\left(s'\vert s, a\right)
    = \frac{\sum_{i: i < k}\mathds{1}\left\{s_{h + 1}^{i} = s' \,\vert\, \left(s_{h}^{i}, a_{h}^{i}\right) = \left(s, a\right)\right\}}{N_{kh}\left(s, a\right)}
\end{align*}
where 
\begin{align*}
    N_{kh}(s, a) = \max\left\{1, \sum_{i: i < k}\mathds{1}\left\{\left(s_{h}^{i} = s, a_{h}^{i} = a\right)\right\}\right\}
\end{align*}
is the total visitation count.\\

There are two main ways of ensuring optimism using model-based algorithms. The first is the model-optimistic approach, which maintains a confidence set around $\hat{P}_{kh}$ that contains $P_h$ with high probability \citep{UCRL2, KLUCRL, UCRL2B}. The second is the value-optimistic approach, which involves directly upper bounding the optimal value function with high probability by adding a bonus to the value function of a policy under the estimated transition density $\hat{P}_{kh}$ \citep{UBEV, UCBVI}. Recent work has shown that all model-based optimistic algorithms have a value-optimistic representation, meaning they all compute a value function of the following form \citep{PB2020}: 
\begin{align}
    \tilde{V}_{h}^{\pi} = \left(H' + 1\right)\land \left(R_{h} + \big\langle \hat{P}_{kh},  \tilde{V}_{h + 1}^{\pi}\big\rangle + \beta_{kh}^{+}\right)\label{eqn: optimistic value function}
\end{align}
where $H' = H - h$ and
\begin{align}
    \beta_{kh}^{+}\left(s, a\right) &= H' \land 
    \left(\frac{B_{1}}{\sqrt{N_{kh}\left(s, a\right)}} + \frac{B_{2}}{N_{kh}\left(s, a\right)}\right)\notag\\
    &= H' \land \beta_{kh}\left(s, a\right) \label{eqn: exploration bonus}
\end{align}
is the exploration bonus and $x\land y = \min\{x, y\}$. Here, $B_{1}$ and $B_{2}$ are algorithm-dependent quantities which may depend on $S,A,H, \log(T)$ or the empirical variance of the optimistic value function. A suitably chosen exploration bonus ensures the computed value function is optimistic with high probability. For our theoretical results to hold, we require the following assumption on the base algorithm.
\begin{assumption}\label{assumption: estimation error bonus}
The exploration bonus upper bounds the estimation error with high probability. Mathematically: $\beta_{kh}^{+}(s, a) \geq \langle (\hat{P}_{kh} - P_{h})(\cdot\vert s, a), V^{*}_{h + 1}(\cdot)\rangle$ for all time-steps, with probability $1 - \delta$.
\end{assumption}

All value-optimistic algorithms explicitly use the estimation error to derive suitable bonuses. Further, model-optimistic algorithms compute bonuses satisfying this assumption implicitly \citep{PB2020}. Therefore, Assumption \ref{assumption: estimation error bonus} allows us to capture a wide range of model-based algorithms.\\

For our analysis, it will be helpful to define an algorithm-dependent variable $C$, which indicates whether the algorithm's bonuses satisfy the following inequality:
\begin{equation}\label{eqn: C}
    \beta_{kh}^{+}(s, a)) < \Big\langle \left(\hat{P}_{kh} - P_{h}\right)\left(\cdot \,\vert s, a \right), \tilde{V}_{h + 1}^{\pi_{k}}(\cdot) \Big\rangle
\end{equation}

for all $s, a, h, k$ with probability $1 - \delta$. Intuitively, $C = 1$ corresponds to a bonuses that sits somewhere between the estimation error and the difference between the expectation of the optimistic value function under the estimated and true transition function. Since these bonuses must sit within a specific (potentially narrow) interval, they are tighter. However, as we will see later, such bonuses come at the expense of lower-order terms. UBEV and UCBVI are algorithms where $C = 1$. Whereas UCRL2, UCRL2B, KL-UCRL and $\chi^2$-UCRL are algorithms with $C = 0$.

\section{Delayed Feedback}\label{sec: delayed feedback}
Under stochastic delays, the feedback from an episode does not return to the base algorithm immediately after the interaction. Instead, it returns at some unknown time in  the future, $k + \tau_{k}$. Here, $\tau_{k}$ denotes the random delay between the agent playing the $k^\text{th}$ episode and the base algorithm receiving the corresponding feedback. Throughout this paper, we make the following assumption about the delays:
\begin{assumption}\label{assumption: delays}
The delays are positive, independent and identically distributed random variables with a finite expected value, $\mathbb{E}[\tau_{k}] < \infty$.
\end{assumption}

The introduction of delays causes the feedback associated with an episode to return at some unknown time in the future, $k + \tau_{k}$. As a result, the base algorithm cannot update its policy using feedback from episode $k$ at the start of episode $k + 1$. Instead, it can only use feedback it has observed, e.g. the feedback associated with episodes $i: i + \tau_{i} < k + 1$.\\

When working with delayed feedback in RL, it is helpful to introduce the observed and missing visitation counters: 
\begin{align}
    N_{kh}'\left(s, a\right) &= \sum_{i: i + \tau_{i} < k}\mathds{1}\left\{\left(s_{h}^{i}, a_{h}^{i}\right) = \left(s, a\right)\right\}\label{eqn: observed visitation counter}\\
    N_{kh}''\left(s, a\right) &= \sum_{i: i + \tau_{i} \geq k}\mathds{1}\left\{\left(s_{h}^{i}, a_{h}^{i}\right) = \left(s, a\right)\right\}.\label{eqn: missing visitation counter}
\end{align}
These are related to the total visitation counter by\begin{equation}
    N_{kh}\left(s, a\right) = N_{kh}'\left(s, a\right) + N_{kh}''\left(s, a\right)\label{eqn: visitation counter relationships}\;.
\end{equation}
When the feedback is delayed, optimistic algorithms can only compute their bonuses and any required estimators using the observed visitation counter. The corresponding value functions are still optimistic, but they contract to the optimal value function more slowly since $N_{kh}(s, a) \geq N_{kh}'(s, a)$.

\subsection{Bounding the Missing Episodes}
In our analysis, it is helpful to bound the number of missing episodes to get an upper bound on the amount of information missing for each state-action-step. This is done in the following lemma.

\begin{restatable}{lemma}{delays}
\label{lemma: delay failure event}
Let $S_{k} = \sum_{i = 1}^{k - 1}\mathds{1}\{i + \tau_{i}\geq k\}$, where $\tau_{1}, \tau_{2}, \cdots \tau_{k - 1} \sim f_{\tau}(\cdot)$ are independent and identically distributed random variables with finite expected value. We define
 \begin{equation*}
    F_{k}^{\tau} = \left\{S_{k} \geq \mathbb{E}\left[\tau\right] +  \log\left(\frac{K\pi}{6\delta'}\right) + \sqrt{2\mathbb{E}\left[\tau\right]\log\left(\frac{K\pi}{6\delta'}\right)}\right\}
 \end{equation*}
to be the failure event for a single $k$. Then, $\mathbb{P}(F_{\tau}) = \mathbb{P}(\cup_{k = 1}^{\infty} F_{k}^{\tau}) \leq \delta'$.
\end{restatable}
\begin{proof}
Firstly, notice that $S_k$ is a sum of Bernoulli random variables, meaning it is subgaussian. Therefore, one can apply Bernstein's inequality to obtain the following upper bound that holds with probability $1 - \delta '$:. 
$$
S_k \leq \mathbb{E}[S_k] + \frac{2}{3}\log\left(\frac{k\pi}{6\delta'}\right) + \sqrt{2\text{Var}\left(S_k\right) \log\left(\frac{k\pi}{6\delta'}\right)} 
$$
The remainder of the proof follows from noticing that $\mathbb{E}[S_k] \leq \sum_{i = 0}^{\infty}\mathbb{P}(\tau > i)$, which is the tail probability function of the delay distribution and is equal to the expected delay. Similarly, one can show that $\text{Var}\left(S_k\right) \leq \mathbb{E}[S_k] \leq \mathbb{E}[\tau]$. Substituting these values into the above inequality gives the result. See Appendix \ref{sec: delay failure event} for a full proof.
\end{proof}
A direct consequence of this lemma is an upper bound on the number of missing episodes $S_k \leq \psi_K^\tau$ for
\begin{align*}
    \psi_{K}^{\tau}\coloneqq \mathbb{E}\left[\tau\right] +  \log\left(\frac{K\pi}{6\delta'}\right) + \sqrt{2\mathbb{E}\left[\tau\right]\log\left(\frac{K\pi}{6\delta'}\right)}
\end{align*}
which holds for all $k \in [K]$ with probability $1 - \delta'$.  Essentially, $\psi_{K}^{\tau}$ allows us to bound the amount of missing information in any given episode due to the delays. 

\section{Meta-Algorithms For Delayed Feedback}

Here, we describe two flexible approaches that allow any base algorithm to handle delayed feedback. Additionally, we prove regret guarantees for both procedures, providing the base algorithm satisfies Assumption \ref{assumption: estimation error bonus}. Regardless of the approach, we utilise the following regret decomposition for optimistic base algorithms that holds for both the delayed and non-delayed settings.

\begin{restatable}{lemma}{decomposition}\label{lemma: regret decomposition}
Under Assumption \ref{assumption: estimation error bonus}, with probability $1 - 4\delta'$, we can upper bound the regret by:
\begin{align*}
\mathfrak{R}_{K} \leq 6\left(H + C\right)\sqrt{T\log\left(\frac{K\pi}{6\delta'}\right)} + 6\sum_{k = 1}^{K}\sum_{h = 1}^{H} \beta_{kh}^{+}\left(s_{h}^{k}, a_{h}^{k}\right) + 6\sum_{k = 1}^{K}\sum_{h = 1}^{H} \frac{3CH^2 S L }{N_{kh}'\left(s_{h}^{k}, a_{h}^{k}\right)}
\end{align*}
where $L = \log\left(S^2 A H \pi^{2}/6\delta'\right)$ and $C$ indicates whether the bonuses of the algorithm satisfy Equation \eqref{eqn: C}.
\end{restatable}
\begin{proof}
See Appendix \ref{sec: missing proof for decomposition}.
\end{proof}

\subsection{Active Updating}\label{sec: active updating}
The first meta-algorithm we propose is \textit{active updating}, which leverages new information by updating as soon as it becomes available. The remainder of this subsection focuses on bounding the regret for model-based optimistic algorithms using active updating, whose pseudo-code is outlined in Algorithm \ref{alg: Active Algorithm}.

\begin{algorithm}[h]\caption{Active Updating}\label{alg: Active Algorithm}
\begin{algorithmic}
    \STATE \textbf{Input.} Base$(N', M')$ (any base algorithm).
    \STATE \textbf{Initialise.} $N' = \{N_{h}'(s, a) = 0\}_{h, s, a}$ and $M' = \{M_{h}'(s, a, s') = 0\}_{h, s, a}$ with 
    $$M_{h}'(s, a, s') \coloneqq \sum_{i: i + \tau_{i} < k} \mathds{1}\{(s_{h}^{i}, a_h^i, s_{h + 1}^i) = (s, a, s')\}$$
    \STATE Compute policy: $\pi_{1} = Base(N', M')$
    \FOR{$k = 1$ {\bfseries to} $K$}
       \IF{$\exists \,i : k - 2 < i + \tau_{i} \leq k - 1$}
        \STATE Update the counters: $N'$ and $M'$.
        \STATE Update the policy: $\pi_{k}$ = $Base(N', M')$
        \ELSE 
        \STATE Reuse previous policy: $\pi_{k} = \pi_{k - 1}$
        \ENDIF
    \STATE An agent samples an episode using policy $\pi_{k}$.
    \ENDFOR
\end{algorithmic}
\end{algorithm}

Base$(N', M')$ is the only input parameter for our algorithm and is the base algorithm. One could view it as a function that takes in the observed number of visits ($N'$) and transitions ($M'$), among other algorithm-dependent hyperparameters, and returns a policy. For the class of optimistic algorithms, the additional hyperparameter is the confidence level, $\delta$.

\begin{theorem}[Active Updating]\label{theorem: active updating}
Under Assumption \ref{assumption: estimation error bonus} and \ref{assumption: delays}, with probability $1 - \delta$, the regret of any model-based algorithm under delayed feedback:
\begin{align*}
 \mathfrak{R}_{K} \lesssim B\sqrt{HSAT} + \max\left\{B, B_{2}, CH^2S\right\}HSA\,\mathbb{E}\left[\tau\right]
\end{align*}
where $\lesssim$ suppresses numeric constants, poly-log and lower order terms, and $B \geq B_{1}$ is a upper bound on the leading-order term in the numerator of the exploration bonus that is a function of $H$ and $S$, and holds for all $(k, h)\in [K]\times[H]$.
\end{theorem}
\begin{proof}
From Lemma \ref{lemma: regret decomposition}, it is clear that we must bound the summation of the bonuses to bound the regret. When there are no delays, one can utilise the fact that the visitation count for $(s, a, h)$ at the start of episode $k + 1$ increases by one if the agent observed $(s, a, h)$ in the $k$-th episode to bound this term. However, this is no longer the case under delayed feedback. Therefore, we introduce the following lemma to bound the delay-dependent visitation counter. 
\begin{restatable}{lemma}{summation}\label{lemma: summation bound}
Let $Z_{T}^{p} = \sum_{k = 1}^{K}\sum_{h = 1}^{H} 1/(N_{kh}'(s_h^k, a_h^k))^p$. Then,
\begin{align*}
    Z_T^p &\leq 
    \begin{cases}
    4\sqrt{HSAT} + 3HSA\psi_{K}^{\tau} & \text{if } p = \frac{1}{2}\\
    2HSA \log\left(8T\right) + HSA\psi_{K}^{\tau}\log(16\psi_{K}^{\tau}) & \text{if } p = 1
    \end{cases}
\end{align*}
with probability $1 - \delta'$.
\end{restatable}
\begin{proof}
To prove the claim, we relate the sum involving the observed visitation counters to a sum involving the total visitation counters. To do so, we artificially introduce it into the summation by multiplying by one:
\begin{align*}
    &Z_T^p = \sum_{k = 1}^{K}\sum_{h = 1}^{H} \left(\frac{N_{kh}'(s_h^k, a_h^k) + N_{kh}''(s_h^k, a_h^k)}{N_{kh}'(s_h^k, a_h^k) N_{kh}(s_h^k, a_h^k)}\right)^p = \sum_{k = 1}^{K}\sum_{h = 1}^{H} \left(\frac{1}{N_{kh}(s_h^k, a_h^k)} + \frac{N_{kh}''(s_h^k, a_h^k)}{N_{kh}'(s_h^k, a_h^k)N_{kh}(s_h^k, a_h^k)} \right)^p
\end{align*}
The term in the numerator of the first line is equivalent to the total visitation counter by the equivalence relation given in Equation \eqref{eqn: visitation counter relationships}. One can handle the first term using standard results from the immediate feedback setting. The remainder of the proof follows from carefully splitting the second term in the sum on the second line into two disjoint sets. Namely, we split the summation using two indicators: $\mathds{1}\{N_{kh}'(s, a) \geq \psi_{K}^{\tau}\}$  and $\mathds{1}\{N_{kh}'(s, a) < \psi_{K}^{\tau}\}$. After a little algebra, we find that we are able to apply results from the immediate feedback setting, which gives the final result. See Appendix \ref{sec: missing proofs for active} for further details.
\end{proof}

For many algorithms, $B_{1}$ depends polynomially on quantities related to the environment, e.g. $H$ and $S$. For such algorithms, a direct application of Lemma \ref{lemma: summation bound} is able to separate the expected delay from the total number of decisions. This is in line with the intuition that the impact of delays are negligible once we have a reasonable model of the environment. However, for algorithms such as for UCRL2B, $\chi^2$-UCRL and UCBVI \citep{UCRL2B, PB2020, UCBVI}:
$$B_{1} = \widetilde{\mathcal{O}}\left(\sqrt{\mathbb{V}_{s'\sim \hat{P}_{kh}(\cdot\,\vert\,s,  a)}\left(\tilde{V}_{h + 1}^{\pi}(s')\right)}\right)$$
Typically, one uses an application of Cauchy-Schwarz to separate the terms involving the variance from those involving the counters, which gives: 
\begin{align*}
 &\sum_{k, h} \sqrt{\frac{\text{Var}_{s'\sim \hat{P}_{h}}(\tilde{V}_{h + 1}(s'))}{N_{kh}'(s_{h}^{k}, a_{h}^{k})}} \leq \sqrt{\sum_{k, h}\text{Var}_{s'\sim \hat{P}_{h}}(\tilde{V}_{h + 1}(s')) \sum_{k, h}\frac{1}{N_{kh}'(s_{h}^{k}, a_{h}^{k})}}
\end{align*}
Lemma \ref{lemma: summation bound} shows that doing so would lead to the delays multiplying the leading order term, as the summation of the variances found underneath the square root is of order $HT$ and multiplies the $HSA\psi_{K}^{\tau}$ that arises from bounding the summation of the observed visitation counter. Setting $B = H/2$ gives us an upper bound for these types of bonuses and avoids this multiplicative dependence.\\

Using a uniform upper bound on $B_{1}$ allows us to handle the remaining summations as follows: 
\begin{align}
    \sum_{k = 1}^{K}\sum_{h = 1}^{H} \beta_{kh}^{+}\left(s_{h}^{k}, a_{h}^{k}\right) + \frac{3CH^2 S L }{N_{kh}'\left(s_{h}^{k}, a_{h}^{k}\right)} 
    &\leq \sum_{k = 1}^{K}\sum_{h = 1}^{H} \frac{B_{1}}{\sqrt{N_{kh}'\left(s_{h}^{k}, a_{h}^{k}\right)}} + \frac{B_{2} + 3CH^2 S L}{N_{kh}'\left(s_{h}^{k}, a_{h}^{k}\right)}\notag\\
    &\leq \sum_{k = 1}^{K}\sum_{h = 1}^{H} \frac{B}{\sqrt{N_{kh}'\left(s_{h}^{k}, a_{h}^{k}\right)}} + \frac{B_{2} + 3CH^2 SL}{N_{kh}'\left(s_{h}^{k}, a_{h}^{k}\right)}\label{eqn: intermediate}
\end{align}
Directly applying Lemma \ref{lemma: summation bound} gives the following upper bound on \eqref{eqn: intermediate}: 
\begin{equation*}
    \eqref{eqn: intermediate} \leq 4B\sqrt{HSAT} + 3BHSA\psi_{K}^{\tau} + 2\left(B_{2} + 3CH^2 SL\right)HSA \left(\log(8T) +  \psi_{K}^{\tau}\log\left(16\psi_{K}^{\tau}\right)\right)
\end{equation*}
Substituting the above upper bound of the terms in Lemma \ref{lemma: regret decomposition} and setting $\delta = 5\delta'$ gives the stated result.
\end{proof}

Table \ref{tab: algorithms} in Section \ref{sec: discussion} presents regret bounds for various optimistic algorithms using active updating under delayed feedback that fit into our framework. Further discussion of the results can be found in Section \ref{sec: discussion}.

\subsection{Lazy Updating}\label{sec: lazy updating}
Instead of updating the policy via the base algorithm as soon as new feedback becomes observable, we now consider waiting. We name the meta-algorithm that employs this technique \emph{lazy updating}. Algorithm \ref{alg: Lazy Algorithm} presents the pseudo-code for this meta-algorithm. 

\begin{algorithm}[h]\caption{Lazy Updating}\label{alg: Lazy Algorithm}
\begin{algorithmic}
    \STATE \textbf{Input.} Base$(N', M', \cdots)$ (any base algorithm) and $\alpha$ (activity parameter).
    \STATE Initialise epoch: $j = 1$ and $k_{j} = 1$.
    \STATE Initialise counters: $N_{kh}'(s, a) = M_{kh}'(s, a, s') = 0$.
    \STATE Compute policy: $\pi_{k_{j}} = Base(N_{kh}', M_{kh}')$
    \FOR{$k = 1$ {\bfseries to} $K$}
        \STATE Update counters, e.g. Equation \eqref{eqn: observed visitation counter}.
       \IF{$\exists \,(s, a, h): N_{kh}'(s, a) \geq (1 + 1/\alpha) N_{k_{j}h}'(s, a)$}
       \STATE Update epoch: $j = j + 1$, $k_{j} = k$
        \STATE Update epoch counter: $N_{k_{j}}(s, a) = N_{kh}'(s, a)$
        \STATE Update the policy: $\pi_{k_{j}}$ = $Base(N_{k_{j}h}', M_{k_{j}h}')$
        \ENDIF
    \STATE An agent samples an episode using policy $\pi_{k_{j}}$.
    \ENDFOR
\end{algorithmic}
\end{algorithm}

Lazy updating works in batches of episodes which we call epochs and denote by $j = 1, 2, \cdots, J$. At the start of the $j$-th epoch, lazy updating uses the base algorithm to compute a policy using all the available information. The meta-algorithm uses this policy in every episode until the next epoch begins. Therefore, each epoch is just a set of episodes where the lazy updating algorithm uses the same policy.\\

A new epoch begins as soon as there is an $(s, a, h)$ whose observed visitation counter reaches $1 + 1/\alpha$ times the observed visits at the start of the epoch, where $\alpha \in [1, \infty)$. Note that $\alpha = 1$ corresponds to the well-known doubling trick from \citet{UCRL2}, and $\alpha>1$ represents more frequent updating. Once the observed visitation counter triggers this condition, a new epoch begins, and the meta-algorithm uses the base algorithm to update the policy. Formally, we start epoch $j + 1$ in episode $k_{j + 1}$, which occurs when:
\begin{equation}\label{eqn: updating rule v1}
    k_{j + 1} 
    = \argmin_{k > k_{j}} \left\{N_{kh}' \geq \left(1 + \frac{1}{\alpha}\right)N_{k_{j} h}'\right\}
    = \argmin_{k > k_{j}} \left\{n_{k_{j}}^{k} \geq \frac{1}{\alpha}N_{k_{j} h}'\right\}
\end{equation}
where 
\begin{align}\label{eqn: within counter}
    n_{kh}^{l}\left(s, a\right) = \sum_{i = k}^{l - 1}\mathds{1}\left\{(s_{h}^{i} = s, a_{h}^{i} = a), \,i + \tau_{i} \leq l \right\}
\end{align}
counts the observed number of visits between episodes $k$ and $l$ for $l > k$. Intuitively, this updating scheme forces the number of samples needed for any particular $(s, a, h)$ to trigger an update to increase exponentially quickly, meaning that the total number of epochs should grow logarithmically in $K$. Lemma \ref{lemma: doubling trick} confirms that this is indeed the case.
\begin{restatable}{lemma}{epochs}
\label{lemma: doubling trick}
For $K \geq SA$ and $\alpha \geq 1$, Algorithm \ref{alg: Lazy Algorithm} ensures that the number of epochs has the following upper bound: 
\begin{align*}
    J \leq \frac{ HSA \log\left(\frac{\alpha K}{SA} + 1\right)}{\log(1 + \frac{1}{\alpha})}
    \end{align*}
\end{restatable}
\begin{proof}
 See Appendix \ref{sec: lazy updating proof} for further details.
\end{proof}

In contrast to active updating, we will later see that the lazy updating scheme lets us bound the summation of the bonuses independently of the delays. This property means we can avoid upper bounding the numerator of the exploration bonus, $B_1$, and get tighter leading order terms in the regret bound of the chosen base algorithm. In the regret analysis, we will utilise the following extension of the classic result by \citet{UCRL2} that illustrates the delay-independence of the bonuses.
\begin{restatable}{lemma}{lazycounters}
\label{lemma: lazy counters}
If $n_0, n_1, \cdots, n_{J}$ are an arbitrary sequence of real-valued numbers satisfying $n_0 \coloneqq 0$ and $0 \leq n_j \leq \frac{1}{\alpha}N_{j - 1}$ with $N_{j - 1} = \max\{1, \sum_{i = 0}^{j - 1}n_{i}\}$ for all $j \leq J$, then
\begin{align*}
    \sum_{j = 1}^{J}\frac{n_{j}}{N_{j - 1}^{p}} \leq \begin{cases}
    \left(\sqrt{2}(1 + \frac{1}{\alpha}) + 1\right)\sqrt{N_{J}} & \text{if } p = \frac{1}{2}\\
    (1 + \frac{1}{\alpha}) + (1 + \frac{1}{\alpha})\log\left(N_{J}\right) & \text{if } p = 1
    \end{cases}
\end{align*}
\end{restatable}
\begin{proof}
We prove the claim for each case using an inductive argument similar to \citet{UCRL2}. See Appendix \ref{sec: lazy updating proof}.
\end{proof}

Using Lemmas \ref{lemma: doubling trick} and \ref{lemma: lazy counters}, we can derive regret bounds for any optimistic base algorithm that satisfies Assumption \ref{assumption: estimation error bonus}.

\begin{restatable}{theorem}{lazy}\label{theorem: lazy updating}
Let $K \geq SA$ and $\alpha \geq 1$. Under Assumption \ref{assumption: estimation error bonus} and \ref{assumption: delays}, with probability $1 - \delta$, the regret of any model-based algorithm under delayed feedback is upper bounded by: 
\begin{align*}
    \mathfrak{R}_{K} \lesssim \left(1 + \frac{1}{\alpha}\right)\hat{\mathfrak{R}}_{K}(Base) + \frac{ H^2 SA \mathbb{E}[\tau]}{\log(1 + \frac{1}{\alpha})}
\end{align*}
where $\hat{\mathfrak{R}}_{K}(Base)$ is an upper bound on the regret of the chosen base algorithm under immediate feedback.
\end{restatable}
\begin{proof}
By optimism and utilising the fact that epochs are disjoint sets of episodes, with probability $1 - \delta'$: 
\begin{align*}
\mathfrak{R}_K &\leq \tilde{\mathfrak{R}}_K  \coloneqq \sum_{k = 1}^{K}\tilde{\Delta}_{1}^{k}\left(s_1^k\right) =  \sum_{j = 1}^{J}\sum_{k = k_{j}}^{k_{j + 1} - 1}\tilde{\Delta}_{1}^{k}\left(s_1^k\right)\\
&\leq HJ + \sum_{j = 1}^{J} \sum_{k = k_{j} + 1}^{k_{j + 1} - 1}\tilde{\Delta}_{1}^{k}\left(s_1^k\right)
\end{align*}
where the final inequality follows from separating the episodes where we update and bounding their contribution to the regret by $HJ$.\\

Handling the remaining summation in the regret bound requires a little more care, which we do by splitting the remaining sum into two sets; episodes with short and long delays. An episode has a short delay if it is played and observed in the same epoch, $\mathds{1}\{k + \tau_k < k_{j + 1}\}$. Otherwise, it has a long delay, $\mathds{1}\{k + \tau_k \geq k_{j + 1}\}$.\\

One can show that the regret of episodes with long delays has the following upper bound:
$$
\sum_{j = 1}^{J} \sum_{k = k_{j} + 1}^{k_{j + 1} - 1}\tilde{\Delta}_{1}^{k}\left(s_1^k\right)\mathds{1}\{k + \tau_k \geq k_{j + 1}\} \leq H \sum_{j = 1}^{J}S_{k_{j + 1}}
$$
Aforementioned, $S_{k} \leq \psi_{K}^{\tau}$ for all $k \leq K$ with probability $1 - \delta'$. Therefore, we can upper bound the regret of episodes with long delays by $HJ\psi_{K}^{\tau}$.\\

All that remains is bounding the regret of episodes with short delays. Applying Lemma \ref{lemma: regret decomposition} to these episodes and re-arranging gives:\footnote{Here, we have omitted lower order terms for brevity.}
\begin{align*}
    \sum_{j = 1}^{J} \sum_{k = k_{j} + 1}^{k_{j + 1} - 1}\tilde{\Delta}_{1}^{k}\left(s_1^k\right)\mathds{1}\{k + \tau_k < k_{j + 1}\} \lesssim \sum_{s, a, h}\sum_{j = 1}^{J} n_{k_{j + 1}}^{k_{j + 1}}(s, a) \beta_{kh}\left(s, a\right)\mathds{1}\{k + \tau_{k} < k_{j + 1}\}
\end{align*}
where we have omitted the state-action-step triples that caused the update from the summation. By construction, all the state-action-step triples satisfy the conditions of Lemma \ref{lemma: lazy counters}. Applying this result to the summation of the bonuses and combining the contributions of the other terms gives the result. See Appendix \ref{sec: lazy regret bound proof} for a full proof of the claim.
\end{proof}

\subsection{Discussion}\label{sec: discussion}
Table \ref{tab: algorithms} presents a selection of algorithms that fit into our framework and their accompanying theoretical guarantees when using the active and lazy updating meta-algorithms to handle delayed feedback. In particular, we see that acting in delayed environments causes an additive increase in regret for almost all combinations of optimistic base algorithms and meta-algorithms considered. This result mirrors what is seen in the bandit setting where algorithms incur an additive regret penalty involving $\mathbb{E}[\tau]$ \citep{PJ2013}.

\begin{center}
\begin{table*}[h!]
\label{tab: algorithms}
\resizebox{\textwidth}{!}{%
    \begin{tabular}{l l l l l l l }
        \\
        \toprule
        \textbf{Base Algorithm} & $C$ & $\hat{\mathfrak{R}}_{K}(\text{Base})$ & \textbf{Active Updating} & \textbf{Lazy Updating}\\
        \midrule
        UBEV \citep{UBEV} & $1$ & $H^{3/2}\sqrt{SAT}$ & $\hat{\mathfrak{R}}_{K}(\text{Base}) + H^3 S^2 A \mathbb{E}[\tau]$ & $(1 + \frac{1}{\alpha})\,\hat{\mathfrak{R}}_{K}(Base)  + \frac{ H^2 SA \mathbb{E}[\tau]}{\log(1 + \frac{1}{\alpha})}$\\
        UCBVI-CH \citep{UCBVI} & $1$ & $H^{3/2}\sqrt{SAT}$ & $\hat{\mathfrak{R}}_{K}(\text{Base}) + H^3 S^2 A \mathbb{E}[\tau]$ & $(1 + \frac{1}{\alpha})\,\hat{\mathfrak{R}}_{K}(Base)  + \frac{ H^2 SA \mathbb{E}[\tau]}{\log(1 + \frac{1}{\alpha})}$\\
        UCRL2 \citep{UCRL2} & $0$ & $H^{3/2}S\sqrt{AT}$ & $\hat{\mathfrak{R}}_{K}(\text{Base}) + H^2 S^{3/2}A \mathbb{E}[\tau]$ & $(1 + \frac{1}{\alpha})\,\hat{\mathfrak{R}}_{K}(Base)  + \frac{ H^2 SA \mathbb{E}[\tau]}{\log(1 + \frac{1}{\alpha})}$\\
        KL-UCRL \citep{KLUCRL} & $0$ & $H^{3/2}S\sqrt{AT}$ & $\hat{\mathfrak{R}}_{K}(\text{Base}) + H^2 S^{3/2}A \mathbb{E}[\tau]$& $(1 + \frac{1}{\alpha})\,\hat{\mathfrak{R}}_{K}(Base)  + \frac{ H^2 SA \mathbb{E}[\tau]}{\log(1 + \frac{1}{\alpha})}$\\
        UCRL2B \citep{UCRL2B} & $0$ & $H\sqrt{S\Gamma AT}$ & $\sqrt{H}\hat{\mathfrak{R}}_{K}(\text{Base}) + H^2 S^2 A \mathbb{E}[\tau]$ & $(1 + \frac{1}{\alpha})\,\hat{\mathfrak{R}}_{K}(Base)  + \frac{ H^2 SA \mathbb{E}[\tau]}{\log(1 + \frac{1}{\alpha})}$\\
        $\chi^2$-UCRL \citep{PB2020} & $0$ & $HS\sqrt{AT}$ & $\sqrt{H}\hat{\mathfrak{R}}_{K}(\text{Base}) + H^2 S^2 A \mathbb{E}[\tau]$ & $(1 + \frac{1}{\alpha})\,\hat{\mathfrak{R}}_{K}(Base)  + \frac{ H^2 SA \mathbb{E}[\tau]}{\log(1 + \frac{1}{\alpha})}$\\
        UCBVI-BF \citep{UCBVI} & $1$ & $H\sqrt{SAT}$ & $\sqrt{H} \hat{\mathfrak{R}}_{K}(\text{Base}) + H^3 S^2 A \mathbb{E}[\tau]$ & $(1 + \frac{1}{\alpha})\,\hat{\mathfrak{R}}_{K}(Base)  + \frac{ H^2 SA \mathbb{E}[\tau]}{\log(1 + \frac{1}{\alpha})}$\\
        \bottomrule
    \end{tabular}}
\caption{A selection of algorithms that fit into our framework and their regret bounds under delayed feedback. Here, $\Gamma \leq S$ denotes a uniform upper bound on the number of reachable states.}
\end{table*}
\end{center} 

For active updating and some base algorithms, we found that the additive delay dependence comes at the price of a penalty to the leading order term in the regret bound. Namely, an extra $\sqrt{H}$. This extra penalty multiplying the leading order term is a feature of the theoretical analysis. Another important factor influencing the impact of the delays when using active updating is the parameter $C$. The penalty for delayed feedback is higher when $C = 1$. The worsened delay dependence for these algorithms is due to the introduction of lower-order terms in the probabilistic analysis under immediate feedback, which allows for tighter bonuses. Unfortunately, these lower-order terms become dependent on the delays in our setting and thus lead to a worse delay dependence.\\

To rectify the undesirable penalty to the leading order terms and the dependence on $C$, we developed an alternative approach called lazy updating, which achieves the same additive delay dependence for all algorithms that fit into our framework with only a logarithmic penalty to the leading order term in the regret bound of the base algorithm under immediate feedback. This approach works by introducing an additional hyperparameter that controls how frequently the base algorithm updates its policy. We denote this hyperparameter by $\alpha$ and name it the activity parameter. Theorem \ref{theorem: lazy updating} indicates that there is a trade-off when selecting $\alpha$. On the one hand, we would like to choose a large value of $\alpha$ to minimise the penalty to the leading order term, which is arises from the slower updating. On the other hand, the penalty introduced by the delays is a strictly increasing function of $\alpha$, making large values undesirable. As $\alpha \rightarrow \infty$, lazy updating tends to active updating; at this limiting value, lazy updating will update as soon as it receives new feedback, just like active updating. Thus, the empirical performance of lazy updating should get closer to active updating as $\alpha$ increases. In Section \ref{sec: experiments}, we demonstrate that this is the case and show that it is possible to get most of the benefits of active updating with a relatively modest value of $\alpha$, which has better worst-case regret bounds in the delayed feedback setting. \\

Comparatively, our work significantly improves the regret bounds for many algorithms in the delayed feedback setting. \citet{DAMDP} presents regret bounds for stochastic MDPs of the form $H^{3/2}S\sqrt{AT} + H^2 S \tau_{\max}$ for all optimistic algorithms. Except for UCRL2 and KL-UCRL, the leading order term in their regret bound is loose in either $H$, $S$ or both. Conversely, the leading order terms in our regret bounds are tight for all algorithms when utilising lazy updating and are only loose by a factor of $\sqrt{H}$ for a few algorithms when utilising active updating. Furthermore, $\mathbb{E}[\tau] \ll \tau_{\max}$ in almost all scenarios. As a result, our regret bounds have a tighter delay dependence. Our algorithms also remove the need for \textit{a-priori} knowledge of the maximal delay.\\

The setting of delayed feedback also generalises the case where only the rewards are delayed. Thus, our theoretical results also hold for this setting if we directly apply active or lazy updating. However, one could do better in this case by realising that it is only the delays impacting the rewards, meaning it is only necessary to apply the meta-algorithms to the estimation of the rewards. We expect the additive penalty to be $HSA\mathbb{E}[\tau]$. Indeed, the improved delay-dependence is due to the fact that learning the expected reward function is an easier task than learning the transitions. We prove that this is indeed the case for UCRL2 algorithm of \cite{UCRL2} in Appendix \ref{sec: missing proof for delayed rewards}.

\section{Experimental Results}\label{sec: experiments}

In this section, we investigate the impact of delayed feedback on the regret of active and lazy updating in the chain environment of \citet{ChainMDP}. Briefly, this environment consists of a sequence of $S$ states arranged side-by-side. The learner starts in the left-most state and has to decide between $A = 2$ actions, head left or right. Each episode consists of $H = S$ decisions and the only state with a reward is the right-most state. Thus, the optimal policy is to head right at every step. Heading left is always successful. However, heading right is successful with probability $1 - 1/S$. If unsuccessful, the learner moves one state to the left. Notably, any inefficient exploration strategy will take at least $2^{S}$ episodes to learn the optimal policy \citep{ChainMDP}.\\

We consider chains with $H = S \in \{5, 10, 20, 30\}$ and use UCBVI-BF as the base algorithm in all of our experiments as it has the best regret guarantees under immediate feedback. For our lazy updating approach, we selected several values for the activity hyperparameter, $\alpha \in \{1, 10, 100\}$. In all our experiments, we set the confidence parameter of the base algorithm so that the regret bounds hold with probability $0.95$. Additionally, we compare our meta-algorithms to the explicit exploration procedure proposed by \citet{DAMDP}. Their procedure requires prior knowledge of the maximum delay, which we provide by generating all the delays before the first episode and taking the maximum. In practice, the maximum delay is often unknown and possibly infinite, making this approach infeasible. \\

Our experiments consider Constant, Geometric, Poisson and Uniform delays. For each of these distributions, we consider the following expected delays: $\mathbb{E}[\tau] \in \{0, 100, 200, 300, 400, 500\}$.\footnote{For the uniformly distributed delays, we set the lower and upper limits to $0$ and $2\mathbb{E}[\tau]$, respectively.} All results are averaged over $30$ independent runs and the shaded regions in all the figures contain 95\% of our empirical results.\\

\begin{figure*}[h!]
    \centering
    \includegraphics[width = \textwidth]{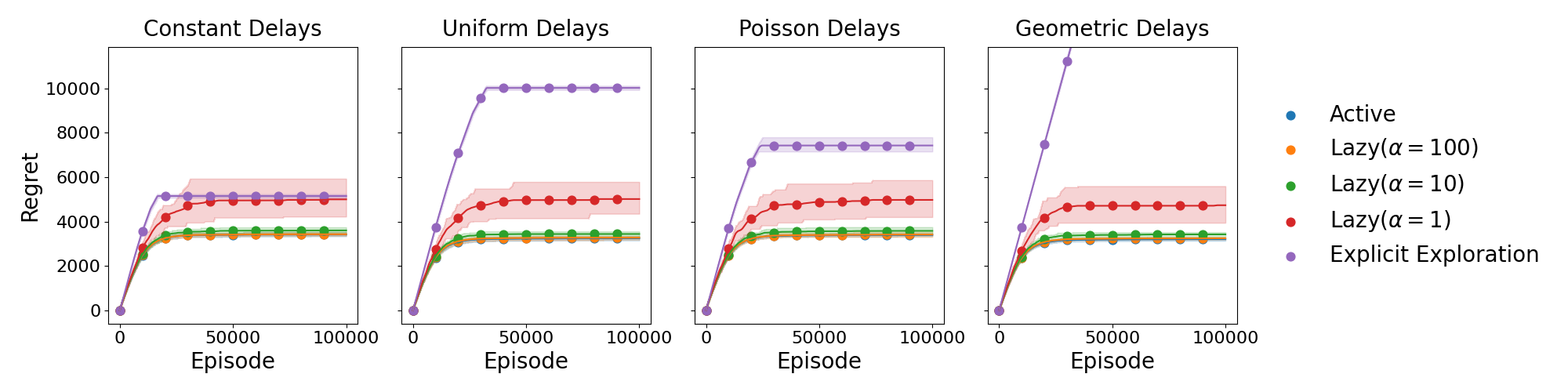}
    \caption{Cumulative Regret $\left(S = 30, \mathbb{E}[\tau] = 100\right)$.}
    \label{fig: chain_regret_100}
\end{figure*}

Figure \ref{fig: chain_regret_100} displays the results for our experiments in the chain environment with $S= 30$ and $\mathbb{E}[\tau] = 100$. The results for the other chain lengths and expected values are in Appendix \ref{sec: additional experiments}. Empirically, active updating achieves the best performance of all three meta-algorithms. However, our experimental results suggest that it is possible to get near identical performance with lazy updating by setting $\alpha$ to be a large enough constant. Both active and lazy updating offer superior performance to the explicit exploration approach of \citet{DAMDP} in all of our experiments, despite their meta-algorithm having prior knowledge of the delays. In some cases, our meta-algorithms have converged to the optimal policy before the explicit exploration procedure finishes; e.g. see Appendix \ref{sec: additional experiments}.\\

\begin{figure*}[h!]
    \centering
    \includegraphics[width = \textwidth]{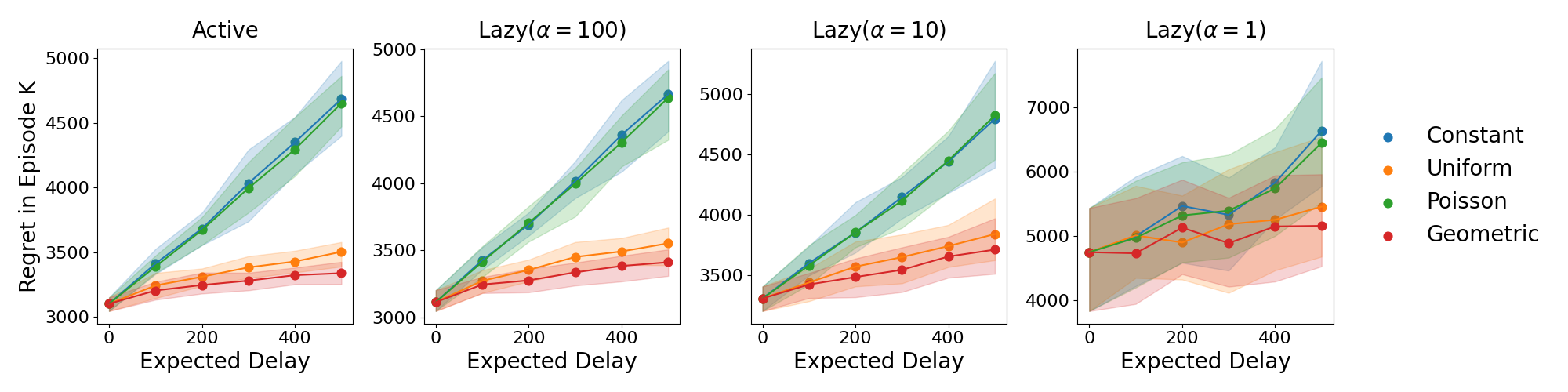}
    \caption{Delay Dependence $\left(S = 30\right)$.}
    \label{fig: chain_delays}
\end{figure*}

Next, we turn to considering the impact of different delay distributions on the regret of our meta-algorithms. Empirically, Figure \ref{fig: chain_delays} shows that the regret penalty of delays at the end of the final episode is linear in the expected delay for active updating and lazy updating, as our theory predicts. For lazy updating, the gradient of this linear relationship decreases with $\alpha$, which is to be expected based on the $\log(1 + 1/\alpha)$ term in the denominator of the delay-dependent terms in our regret bounds. Interestingly, lazy updating with $\alpha = 1$ is the most robust to the delay distribution. We believe that this is due to forcing the base algorithm to wait for long periods of time between updates. Intuitively, if the epochs are long enough, most information within an epoch will be received before an update, leading to little loss of information. Investigating this further is an interesting avenue for future work.

\section{Conclusion}
In this paper, we provide two generic meta-algorithms that can extend any episodic reinforcement learning base algorithm to the setting of delayed feedback. Under mild assumptions on the algorithm and the delays, we show that both maintain the sub-linear theoretical guarantees of the chosen base algorithm and provide good empirical performance, regardless of the delay distribution. These first positive results for stochastically delayed feedback in episodic reinforcement learning prove that the penalty for delays is an additive term involving the expected delay that is independent of the number of episodes. This additive penalty matches what is seen in the multi-armed bandit setting, despite the additional complexities of the reinforcement learning problem.\\

Our framework is broad enough to cover the theoretically successful class of optimistic model-based algorithms, and many existing algorithms fit into our framework. However, we believe that both updating procedures could be used for a wider class of base algorithms. For example, model-free optimistic algorithms and posterior sampling \citep{QLearning, ChainMDP}. Extending our analyses to cover these algorithms is left to future work.
\newpage
\bibliographystyle{plainnat}
\bibliography{references}

\newpage
\begin{appendix}
\section{Missing Proofs}
\subsection{Bounding the Missing Episodes}\label{sec: delay failure event}
An important aspect in our proofs is to bound the amount of missing information. Since we see only one state-action pair per step of an episode, an upper bound on the missing visitation counter is simply the number of missing episodes. Lemma \ref{lemma: delay failure event} bounds the number of missing episodes with high probability and only requires the delays have a finite expected value. 

\delays*
\begin{proof}
By definition, the summation involves a sequence of independent indicator random variables. Considering its expectation reveals that: 
\begin{align*}
    \mathbb{E}\left[S_k\right] &=  \sum_{i = 1}^{k - 1}\mathbb{E}\left[\mathds{1}\left\{i + \tau_{i} \geq k\right\}\right] =  \sum_{i = 1}^{k - 1}\mathbb{P}\left[\mathds{1}\left\{i + \tau_{i} \geq k\right\}\right] = \sum_{i = 1}^{k - 1}\mathbb{P}\left[\,\tau_{k - i} > i\,\right] = \sum_{i = 0}^{k - 2} \mathbb{P}\left[\,\tau_{k - i + 1} > i\,\right]\\
    &\leq \sum_{i = 0}^{\infty} \mathbb{P}\left[\,\tau > i\,\right] = \sum_{i = 0}^{\infty} \sum_{j = i + 1}^{\infty} \mathbb{P}\left[\,\tau = j\,\right] = \sum_{j = 1}^{\infty} \sum_{i = 0}^{j - 1} \mathbb{P}\left[\,\tau = j\,\right]= \sum_{j = 1}^{\infty} j\,\mathbb{P}\left[\,\tau = j\,\right]\\
    &= \mathbb{E}\left[\tau\right].
\end{align*}
Next, looking at its variance reveals that: 
\begin{align*}
    \text{Var}\left(S_k\right) &= \sum_{i = 1}^{k - 1}\text{Var}\left(\mathds{1}\left\{i + \tau_{i} \geq k\right\}\right) = \sum_{i = 1}^{k - 1}\mathbb{E}\left[\left(\mathds{1}\left\{i + \tau_{i} \geq k\right\} - \mathbb{E}\left[\mathds{1}\left\{i + \tau_{i} \geq k\right\}\right]\right)^2\right]\\
    &\leq \sum_{i = 1}^{k - 1} \mathbb{E}\left[\mathds{1}\left\{i + \tau_{i}\geq k\right\}^2\right] = \sum_{i = 1}^{k - 1} \mathbb{E}\left[\mathds{1}\left\{i + \tau_{i}\geq k\right\}\right] = \mathbb{E}\left[S_{k}\right]\\
    &\leq \mathbb{E}\left[\tau\right]
\end{align*}
By Bernstein's inequality, we have that: 
\begin{align*}
    \mathbb{P}\left(S_{k} - \mathbb{E}\left[S_{k}\right] \geq \epsilon\right) \leq \exp\left(-\frac{ \epsilon^{2}}{\text{Var}\left(S_{k}\right) + \frac{\epsilon}{3}}\right) = \frac{6\delta'}{\left(k\pi\right)^2}
\end{align*}
Rearranging the above reveals that: 
\begin{align*}
    \epsilon \leq \frac{1}{3}\log\left(\frac{\left(k\pi\right)^2}{6\delta'}\right) + \sqrt{\text{Var}\left(S_{k}\right)\log\left(\frac{\left(k\pi\right)^2}{6\delta'}\right)}\leq \frac{2}{3}\log\left(\frac{k\pi}{6\delta'}\right) + \sqrt{2\mathbb{E}\left[\tau\right]\log\left(\frac{k\pi}{6\delta'}\right)}
\end{align*}
Since $k \leq K$, we have that: 
 \begin{align*}
    \mathbb{P}\left(F_{k}^{\tau}\right) = \mathbb{P}\left(S_{k} - \mathbb{E}\left[\tau\right] \geq \frac{2}{3}\log\left(\frac{K\pi}{6\delta'}\right) + \sqrt{2\mathbb{E}\left[\tau\right]\log\left(\frac{K\pi}{6\delta'}\right)}\right) \leq \frac{6\delta'}{(k\pi)^2}
\end{align*}
By Boole's inequality, we have that: 
\begin{align*}
    \mathbb{P}\left(\bigcup_{k = 1}^{\infty} F_{k}^{\tau}\right) & \leq \sum_{k = 1}^{\infty} \mathbb{P}\left(F_{k}^{\tau}\right) = \frac{6\delta'}{\pi^2}\sum_{k = 1}^{\infty}\frac{1}{k^2} = \delta'
\end{align*}
as required. 
\end{proof}

\subsection{Missing Proofs for Active Updating}\label{sec: missing proofs for active}
Lemma \ref{lemma: regret decomposition} (the regret decomposition) and Equation \eqref{eqn: exploration bonus} (the form of the exploration bonuses) reveal that the summation of the counters is an important quantity in determining the regret of an optimistic algorithm. Whenever $\tau_{k} = 0$ for all $k \leq K$, e.g. immediate feedback, we can use standard results that utilise the fact the counters increase by one between successive plays of a state-action pair at a given step. 

\begin{lemma}\label{lemma: p-series}
Let $Z_n^p = \sum_{n = 0}^{N} 1/(1\lor n)^p$. Then, $Z_n^p$ has the following upper bound: 
\begin{align*}
    Z_n^p \leq 
    \begin{cases}
    2\sqrt{N} & \text{if } p \in \frac{1}{2}\\
    \log\left(8N\right) & \text{if } p = 1
    \end{cases}
\end{align*}
for $p = 1/2$ and $p = 1$. 
\end{lemma}
\begin{proof}
Removing the first two terms from the summation and upper bounding the remaining terms by an integral gives:
\begin{align*}
    Z_n^p &= 2 + \sum_{n = 2}^{N}\frac{1}{n^p}\leq 2 + \int_{1}^{N} \frac{1}{n^p} dn \leq 2 + 
    \begin{cases}
    2\sqrt{N} - 2 & \text{if } p \in \frac{1}{2}\\
    \log\left(N\right) & \text{if } p = 1
    \end{cases}\\
    &\leq  
    \begin{cases}
    2\sqrt{N} & \text{if } p \in \frac{1}{2}\\
    \log\left(8N\right) & \text{if } p = 1
    \end{cases}
\end{align*}
as required.
\end{proof}

When $\tau_{k}$ is random, the observed visitation counter need not increase by one between successive plays of the same state-action-step. Instead, the counter only increases by one (or more in some cases) after a random number of episodes. In the worst-case scenario, the counter will remain constant between playing and observing the feedback associated with a specific state-action-step. Thus, the standard techniques no longer apply, and we must find another way to bound the summation of counters than can remain unchanged for numerous episodes due to the delays. We do this by relating the summation involving the observed visitation counter to one involving the total visitation counter, thereby splitting the terms affected by the delays from those that are not.

\summation*
\begin{proof}
Unless otherwise stated, we let: $N_{kh}(s, a) = 1\lor N_{kh}(s, a)$ and $N_{kh}'(s, a) = 1\lor N_{kh}'(s, a)$ for notational convenience. First, we use the relationships between the observed, missing and total visitation counters to split the summation into two parts. To do so, in a similar manner to \citet{DAMDP}, we start by artificially introducing the total visitation counter:
\begin{align*}
    Z_{T}^{p}&=\sum_{k = 1}^{K}\sum_{h = 1}^{H} \left(\frac{1}{N_{kh}'\left(s_{h}^{k}, a_{h}^{k}\right)}\right)^{p} = \sum_{k, h}\left(\frac{1}{N_{kh}\left(s_{h}^{k}, a_{h}^{k}\right)}\right)^{p} \left(\frac{N_{kh}\left(s_{h}^{k}, a_{h}^{k}\right)}{N_{kh}'\left(s_{h}^{k}, a_{h}^{k}\right)}\right)^{p}
\end{align*}
From Equation \eqref{eqn: visitation counter relationships}, $N_{kh}(s, a) = N_{kh}'(s, a) + N_{kh}''(s, a)$, for any $(s, a, h)\in \mathcal{S}\times \mathcal{A}\times [H]$. Consequently,
\begin{align*}
    Z_{T}^{p} &\leq \underbrace{\sum_{k, h}\left(\frac{1}{N_{kh}\left(s_{h}^{k}, a_{h}^{k}\right)}\right)^{p}}_{(i)} + \underbrace{\sum_{k, h}\left(\frac{1}{N_{kh}\left(s_{h}^{k}, a_{h}^{k}\right)} \frac{N_{kh}''\left(s_{h}^{k}, a_{h}^{k}\right)}{N_{kh}'\left(s_{h}^{k}, a_{h}^{k}\right)}\right)^{p}}_{(ii)},
\end{align*}
since $(1+x)^p \leq 1 + x^p$ for $p = 1/2$ and $p = 1$ and any $x>0$. Term $(i)$ is the summation of the total visitation counter. Thus, Lemma \ref{lemma: p-series} applies.\\

Bounding $(ii)$ requires more care, as it involves the observed and missing visitation counters. Recall that the algorithm plays one state-action pair at each step in every episode. Thus, the missing visitation counter is upper bounded by the number of missing episodes: $N_{kh}''(s, a) \leq S_k$. Lemma \ref{lemma: delay failure event} bounds the number of missing episodes: with probability $1 - \delta'$, $S_{k} \leq \psi_{K}^{\tau}$ across all $k \in \mathbb{Z}^{+}$. Splitting $(ii)$ using the observed visitation counts and the upper bound on $S_k$ gives: 
\begingroup
\allowdisplaybreaks
\begin{align*}
    (ii) &\leq \sum_{k, h}\left(\frac{\mathds{1}\left\{N_{kh}'\left(s_{h}^{k}, a_{h}^{k}\right) \geq \psi_{K}^{\tau}\right\}\psi_{K}^{\tau}}{N_{kh}\left(s_{h}^{k}, a_{h}^{k}\right)N_{kh}'\left(s_{h}^{k}, a_{h}^{k}\right)} \right)^{p} + \sum_{k, h}\left(\frac{\mathds{1}\left\{N_{kh}'\left(s_{h}^{k}, a_{h}^{k}\right) \leq \psi_{K}^{\tau}\right\}\psi_{K}^{\tau}}{N_{kh}\left(s_{h}^{k}, a_{h}^{k}\right) N_{kh}'\left(s_{h}^{k}, a_{h}^{k}\right)} \right)^{p}\nonumber\\
    &\leq \underbrace{\sum_{k, h}\left(\frac{\mathds{1}\left\{N_{kh}'\left(s_{h}^{k}, a_{h}^{k}\right) \geq \psi_{K}^{\tau}\right\}}{N_{kh}\left(s_{h}^{k}, a_{h}^{k}\right)} \right)^{p}}_{(ii.a)} + \underbrace{\sum_{k, h}\left(\frac{\mathds{1}\left\{N_{kh}'\left(s_{h}^{k}, a_{h}^{k}\right) \leq \psi_{K}^{\tau}\right\}\psi_{K}^{\tau}}{N_{kh}\left(s_{h}^{k}, a_{h}^{k}\right)N_{kh}'\left(s_{h}^{k}, a_{h}^{k}\right)}\right)^{p}}_{(ii.b)}
\end{align*}
\endgroup
The last inequality follows since for the first sum,
$N_{kh}'(s, a) \geq \psi_{K}^{\tau}$.\\

Clearly, $(ii.a) \leq (i)$, as it is a summation over a subset of all the episodes. Using \eqref{eqn: visitation counter relationships}, it is possible to rewrite the indicator in the remaining term as: $\mathds{1}\{N_{kh}(s, a) - N_{kh}''(s, a) \leq \psi_{K}^{\tau}\}$, for any $(s, a, h)\in \mathcal{S}\times \mathcal{A}\times [H]$. Further, $N_{kh}''(s, a) \leq \psi_{K}^{\tau}$ and $N_{kh}'(s, a) \geq 1$. Therefore,  
\begin{align*}
    (ii.b) &\leq (\psi_{K}^{\tau})^{p} \sum_{k, h} \left(\frac{\mathds{1}\left\{N_{kh}\left(s_{h}^{k}, a_{h}^{k}\right)  \leq 2\psi_{K}^{\tau}\right\}}{N_{kh}\left(s_{h}^{k}, a_{h}^{k}\right)}\right)^{p}\\
    &\leq (\psi_{K}^{\tau})^{p}\sum_{s, a, h}\sum_{n = 0}^{2\psi_{K}^{\tau}}\frac{1}{(1\lor n)^p}
\end{align*}
Lemma \ref{lemma: p-series} gives an upper bound of $\sum_{n = 0}^{N}1/(1\lor n)^p$. Summing this upper bound over all state-action-step triples gives: 
\begin{align*}
    (ii.b) \leq \begin{cases}
    3HSA\psi_{K}^{\tau} & \text{if } p = \frac{1}{2}\\
    HSA\psi_{K}^{\tau}\log\left(16\psi_{K}^{\tau}\right) & \text{if } p = 1
    \end{cases}
\end{align*}
Therefore: 
\begin{align*}
    Z_{T}^{p} &\leq 2A +  B.2 \\
    &\leq 
    \begin{cases}
    4\sqrt{HSAT} + 3HSA\psi_{K}^{\tau} & \text{if } p = \frac{1}{2}\\
    HSA\left(2\log\left(8T\right) + \psi_{K}^{\tau}\log\left(16\psi_{K}^{\tau}\right)\right) & \text{if } p = 1
    \end{cases}
\end{align*}
as required.
\end{proof}

\subsection{Missing Proofs for Lazy Updating}\label{sec: lazy updating proof}

When using active updating, we prove that the bound on the counts depends on the delay. However, we can mitigate this delay-dependence by taking a slower approach to updating, providing that the number of epochs is bounded and the counts between epochs satisfy certain constraints outlined in Section \ref{sec: lazy updating}.  

\epochs*
\begin{proof}
In this proof, we extend arguments from the standard doubling trick of \citet{UCRL2} so that the learner can update more frequently. Firstly, we recall the definition of the observed visitation counter:\footnote{We move the subscript denoting the step into the bracket for notational convenience} 
\begin{align*}
    N_{k}'\left(s, a, h\right) &= \sum_{i = 1}^{k - 1}\mathds{1}\left\{\left(s_{h}^{i}, a_{h}^{i}\right) = \left(s, a\right), i + \tau_{i} < k\right\}
\end{align*}
and the updating rule for $j\geq 1$:
\begin{align*}
    k_{j + 1} = \argmin_{k > k_{j}} \left\{\exists s, a, h: N_{k}'(s, a, h) \geq \left(1 + \frac{1}{\alpha}\right)N_{k_{j}}'(s, a, h)\right\}
\end{align*}

Now, we define a counter that counts the observed number of visits between two episodes:
\begin{align*}
    n_{k}^{l}\left(s, a, h\right) &= \sum_{i = 1}^{l - 1}\mathds{1}\left\{\left(s_{h}^{i}, a_{h}^{i}\right) = \left(s, a\right), k \leq i + \tau_{i} < l\right\}
\end{align*}

Direct computation allows us to relate the observed visitation counter at the start of the $(j + 1)$-th epoch to the sum of the observed visitation counts within each of the previous epochs: 

\begin{align*}
    N_{k_{j + 1}}'\left(s, a, h\right) &= \sum_{i = 1}^{k_{j + 1} - 1}\mathds{1}\left\{\left(s_{h}^{i}, a_{h}^{i}\right) = \left(s, a\right), i + \tau_{i} < k\right\}\\
    &= \sum_{l = 1}^{j}\sum_{i = k_{l}}^{k_{l + 1} - 1}\mathds{1}\left\{\left(s_{h}^{i}, a_{h}^{i}\right) = \left(s, a\right), i + \tau_{i} < k\right\}\\
    &= \sum_{l = 1}^{j}\sum_{i = 1}^{k_{j + 1}}\mathds{1}\left\{\left(s_{h}^{i}, a_{h}^{i}\right) = \left(s, a\right), k_{l} \leq i + \tau_{i} < k\right\}\\
    &= \sum_{l = 1}^{j} n_{k_{l}}^{k_{l + 1}}\left(s, a, h\right)
\end{align*}
where the second equality follows from the fact that an epoch is a disjoint set of episodes and the final equality follows from the definition of the between episodes visitation counter. From the above, it is easy to see that

\begin{align*}
    N_{k_{j + 1}}'\left(s, a, h\right)
    = n_{k_{j}}^{k_{j + 1}}\left(s, a, h\right) + \sum_{l = 1}^{j - 1} n_{k_{l}}^{k_{l + 1}}\left(s, a, h\right) = n_{k_{j}}^{k_{j + 1}}\left(s, a, h\right) +  N_{k_{j}}'\left(s, a, h\right)
\end{align*}
Thus, we can re-write the updating rule using the within episode counter as:
\begin{align*}
    k_{j + 1} &= \argmin_{k > k_{j}} \left\{\exists s, a, h: N_{k}'(s, a, h) \geq \left(1 + \frac{1}{\alpha}\right)N_{k_{j}}'(s, a, h)\right\}\\
    &= \argmin_{k > k_{j}} \left\{\exists s, a, h: N_{k}'(s, a, h) - N_{k_{j}}'(s, a, h) \geq \frac{1}{\alpha}N_{k_{j}}'(s, a, h)\right\}\\
    &= \argmin_{k > k_{j}} \left\{\exists s, a, h: n_{k_{j}}^{k_{j + 1}}\left(s, a, h\right) \geq \frac{1}{\alpha}N_{k_{j}}'(s, a, h)\right\}
\end{align*}
providing that we have seen the state-action-step at least once.\footnote{We handle the case for the epochs where the observed visitation count is zero later on in the proof.} Therefore, at the end of each epoch there is a state-action-step with $n_{k_{j}}^{k_{j + 1}}\left(s, a, h\right) \geq N_{k_{j}}'(s, a, h)/\alpha$.

Suppose $N_{(K + 1) h}'(s, a) > 0$ for a fixed $(s, a, h)\in \mathcal{S}\times \mathcal{A}\times [H]$. Define $J(s, a, h)$ as the number of epochs with $n_{k_{j}}^{k_{j + 1}}\left(s, a, h\right) \geq N_{k_{j} h}'(s, a)/\alpha$. Or, equivalently, it is the number of epochs with $N_{k_{j + 1}}\left(s, a, h\right) \geq (1 + 1/\alpha)N_{k_{j}}'(s, a, h)$. Then,
\begin{align*}
    N_{K + 1}'\left(s, a, h\right) &= \sum_{j = 1}^{J}  n_{k_{j}}^{k_{j + 1}}\left(s, a, h\right)\\
    &\geq 1 + \sum_{j : n_{k_{j}}^{k_{j + 1}}\left(s, a, h\right) \geq N_{k_{j} }'(s, a, h)/\alpha}n_{k_{j}}^{k_{j + 1}}\left(s, a, h\right)\\
    &\geq 1 + \frac{1}{\alpha}\sum_{j : n_{k_{j}}^{k_{j + 1}}\left(s, a, h\right) \geq N_{k_{j}}'(s, a, h)/\alpha}N_{k_{j}}'(s, a, h)\\
    &\geq 1 + \frac{1}{\alpha}\sum_{j = 1}^{J(s, a, h)}\left(1 + \frac{1}{\alpha}\right)^{j}
\end{align*}

The first inequality follows from focusing only on the epochs where we update due to $(s, a, h)$, where the $+ 1$ accounts for the first update due to the observing the given state-action-step triple. The second inequality follows from the condition in the subscript of the summation, e.g. we are updating due to $(s, a, h)$. The final inequality follows from the definition of how we trigger updates and because we update $J(s, a, h)$ times due to $(s,a,h)$. Since $\alpha \in [1, \infty)$, Lemma \ref{lemma: geometric lazy} applies. Rearranging terms reveals that: 
\begin{align*}
    \sum_{j = 1}^{J(s, a, h)}\left(1 + \frac{1}{\alpha}\right)^{j} \geq \left(1 + \frac{1}{\alpha}\right)^{J(s, a, h) + 1} - \left(1 + \frac{1}{\alpha}\right)
\end{align*}
Therefore, for $N_{K + 1}'(s, a, h) > 0$:
\begin{align*}
    N_{K + 1}'\left(s, a, h\right) &\geq 1 - \frac{1}{\alpha}\left(1 + \frac{1}{\alpha}\right) + \frac{1}{\alpha}\left(1 + \frac{1}{\alpha}\right)^{J(s, a, h) + 1} >  \frac{1}{\alpha}\left(1 + \frac{1}{\alpha}\right)^{J(s, a, h) + 1} - \frac{1}{\alpha}\left(1 + \frac{1}{\alpha}\right)
\end{align*}
If $N_{K + 1}'(s, a, h) = 0$ it follows we never update due to this state-action-step triple, which means that $J(s, a, h) = 0$ too. Plugging this into the above expression reveals that:
$$
N_{K + 1}'(s, a, h) = \frac{1}{\alpha}\left(1 + \frac{1}{\alpha}\right)^{J(s, a, h) + 1} - \frac{1}{\alpha}\left(1 + \frac{1}{\alpha}\right) = 0
$$
Thus, for all possible values of the observed visitation counter, we have that:
\begin{align*}
    N_{K + 1}'\left(s, a, h\right) &\geq \frac{1}{\alpha}\left(1 + \frac{1}{\alpha}\right)^{J(s, a, h) + 1} - \frac{1}{\alpha}\left(1 + \frac{1}{\alpha}\right)
\end{align*}

Using the above inequality, we have that
\begin{align*}
    T &= \sum_{s, a, h} N_{K + 1}(s, a, h)\\
    &\geq \sum_{s, a, h} N_{K + 1}'(s, a, h)\\
    &\geq \sum_{s, a, h}\left(\frac{1}{\alpha}\left(1 + \frac{1}{\alpha}\right)^{J(s, a, h) + 1} - \frac{1}{\alpha}\left(1 + \frac{1}{\alpha}\right)\right)\\
    &= -\frac{HSA}{\alpha}\left(1 + \frac{1}{\alpha}\right) + \sum_{s, a, h}\frac{1}{\alpha}\left(1 + \frac{1}{\alpha}\right)^{J(s, a, h) + 1}\\
    &\geq -\frac{HSA}{\alpha}\left(1 + \frac{1}{\alpha}\right) + \frac{HSA}{\alpha}\left(1 + \frac{1}{\alpha}\right)^{\frac{HSA + \sum_{s, a, h}J(s, a, h)}{HSA}}\tag{Jensen's inequality}\\
    &\geq -\frac{HSA}{\alpha}\left(1 + \frac{1}{\alpha}\right) + \frac{HSA}{\alpha}\left(1 + \frac{1}{\alpha}\right)^{\frac{J}{HSA}}
\end{align*}
where the final line follows from the fact that $J \leq HSA + \sum_{s, a, h}J(s, a, h)$ because we may or may not visit every state-action-step. Rearranging this gives:
\begin{align*}
    \frac{T\alpha}{HSA} + 1 \geq \left(1 + \frac{1}{\alpha}\right)^{\frac{J}{HSA}}
\end{align*}
Taking logs of both sides and rearranging one last time gives: 
\begin{align*}
    J&\leq HSA \log_{1 + 1/\alpha}\left(\frac{T\alpha}{HSA} + 1 \right)\\
    &= \frac{ HSA \log\left(\frac{T\alpha}{HSA} + 1 \right)}{\log(1 + \frac{1}{\alpha})}\\
    &=  \frac{ HSA \log\left(\frac{K \alpha + SA}{SA}\right)}{\log(1 + \frac{1}{\alpha})}
\end{align*}
as required. 
\end{proof}

\lazycounters*
\begin{proof}
We prove the claim via induction in a similar manner to \cite{UCRL2}. First, consider the case where $p = 1/2$. Suppose 
\begin{align*}
    \sum_{j = 1}^{J - 1}n_{j} \leq 1 &\implies N_{1} = N_2 = \cdots = N_{J - 1} = 1 \tag{$N_{j - 1} = \max\{1, \sum_{i = 0}^{j - 1} n_i$\}}\\
    &\implies n_{J} \in [0, N_{J - 1}] = \left[0, \frac{1}{\alpha}\right]
\end{align*}
Then, 
\begin{align*}
    \sum_{j = 1}^{J}\frac{n_{j}}{\sqrt{N_{j - 1}}} = \sum_{j = 1}^{J} n_{j} = n_J +  \sum_{j = 1}^{J - 1} n_{j} \leq \frac{1}{\alpha} + 1 \leq c \sqrt{N_J} \tag{For $c \geq 1 + 1/\alpha$}
\end{align*}
because $N_J \geq 1$. The above is our base case and covers us as long as $\sum_{j = 1}^{J - 1} n_{j} \leq 1$ e.g., when $J = 1$ due to $n_0 \coloneqq 1$. Now, we assume the above holds for $\sum_{j = 1}^{J - 1}n_{j} > 1$: 
\begin{align*}
    \sum_{j = 1}^{J - 1}\frac{n_{j}}{\sqrt{N_{j - 1}}} \leq c \sqrt{N_{J - 1}}
\end{align*}
Finally, we prove the claim holds for $J$:
\begin{align*}
    \sum_{j = 1}^{J}\frac{n_{j}}{\sqrt{N_{j - 1}}} &= \frac{n_{J}}{\sqrt{N_{J - 1}}} + \sum_{j = 1}^{J - 1}\frac{n_{j}}{\sqrt{N_{j - 1}}}\\
    &\leq  \frac{n_{J}}{\sqrt{N_{J - 1}}} + c \sqrt{N_{J - 1}} \tag{Induction Hypothesis}\\
    &= \sqrt{\left(\frac{n_{J}}{\sqrt{N_{J - 1}}} + c \sqrt{N_{J - 1}}\right)^2}\\
    &= \sqrt{\frac{n_J^2}{N_{J - 1}} + 2c n_J + c^2 N_{J - 1} }\\
    &\leq \sqrt{\frac{1}{\alpha} n_J + 2c n_J + c^2 N_{J - 1} } \tag{As $n_J \in [0, N_{J - 1}/\alpha$}\\
    &\leq \sqrt{n_J + 2c n_J + c^2 N_{J - 1} } \tag{As $\alpha \geq 1$}\\
    &= \sqrt{\left(1 + 2c\right) n_J + c^2 N_{J - 1}}\\
    &\leq c\sqrt{n_J + N_{J - 1}} \tag{Pick $c: c^2 \geq 1 + 2c$}\\
    &= c \sqrt{N_J}
\end{align*}
where the final inequality follows from the fact that $\sum_{j = 1}^{J - 1}n_{j} > 1 \implies N_{J} = n_J + N_{J - 1}$. All that remains is selecting $c$. Using the quadratic formula to find the roots of $c^2 - 2c - 1 = 0$, one can deduce that selecting:
$$
c = 1 + \sqrt{2}\left(1 + \frac{1}{\alpha}\right)
$$
satisfies $c \geq 1 + 1/\alpha$ and 

\begin{align*}
    c^2 &= 1 + 2\sqrt{2}\left(1 + \frac{1}{\alpha}\right) + 2\left(1 + \frac{1}{\alpha}\right)^2\\
    &\geq 1 + 2\sqrt{2}\left(1 + \frac{1}{\alpha}\right) + \tag{$\alpha \geq 1$}\\
    &= 1 + 2\left(1 + \sqrt{2}\left(1 + \frac{1}{\alpha}\right)\right)\\
    &= 1 + 2c
\end{align*}
giving the required result. All that remains is to prove the claim for $p = 1$. Similarly to before, suppose:
\begin{align*}
    \sum_{j = 1}^{J - 1}n_{j} \leq 1 &\implies N_{1} = N_2 = \cdots = N_{J - 1} = 1 \tag{$N_{j - 1} = \max\{1, \sum_{i = 0}^{j - 1} n_i$\}}\\
    &\implies n_{J} \in [0, N_{J - 1}] = \left[0, \frac{1}{\alpha}\right]
\end{align*}
Then, 
\begin{align*}
    \sum_{j = 1}^{J}\frac{n_{j}}{N_{j - 1}} = \sum_{j = 1}^{J} n_{j} = n_J +  \sum_{j = 1}^{J - 1} n_{j} \leq \frac{1}{\alpha} + 1 \leq \left( 1 + \frac{1}{\alpha}\right) + \left( 1 + \frac{1}{\alpha}\right) \log(N_{J})
\end{align*}
because $N_J \geq 1$. The above is our base case and covers us as long as $\sum_{j = 1}^{J - 1} n_{j} \leq 1$ e.g., when $J = 1$ due to $n_0 \coloneqq 1$. Now, we assume the above holds for $\sum_{j = 1}^{J - 1}n_{j} > 1$: 
\begin{align*}
    \sum_{j = 1}^{J - 1}\frac{n_{j}}{N_{j - 1}} \leq \left( 1 + \frac{1}{\alpha}\right) + \left( 1 + \frac{1}{\alpha}\right) \log(N_{J - 1})
\end{align*}
Finally, we prove the claim holds for $J$: 
\begin{align*}
    \sum_{j = 1}^{J}\frac{n_{j}}{N_{j - 1}} &= \frac{n_{J}}{N_{J - 1}} + \sum_{j = 1}^{J - 1}\frac{n_{j}}{N_{j - 1}}\\
    &\leq \frac{n_{J}}{N_{J - 1}} + \left( 1 + \frac{1}{\alpha}\right) + \left( 1 + \frac{1}{\alpha}\right) \log(N_{J - 1}) \tag{Induction Hypothesis}\\
    &\leq \left( 1 + \frac{1}{\alpha}\right) \log\left(\frac{n_{J}}{N_{J - 1}} + 1\right) + \left( 1 + \frac{1}{\alpha}\right) + \left( 1 + \frac{1}{\alpha}\right) \log(N_{J - 1})\\
    &= \left( 1 + \frac{1}{\alpha}\right) + \left( 1 + \frac{1}{\alpha}\right) \log\left(N_{J - 1}\left(\frac{n_{J}}{N_{J - 1}} + 1\right)\right)\\
    &= \left( 1 + \frac{1}{\alpha}\right) + \left( 1 + \frac{1}{\alpha}\right) \log\left(n_{J} + N_{J - 1}\right)\\
    &= \left( 1 + \frac{1}{\alpha}\right) + \left( 1 + \frac{1}{\alpha}\right) \log\left(N_{J}\right)
\end{align*}
where the final inequality follows from the fact that $n_j/N_{j - 1} \in [0, 1]$ for all $j \leq J$.
\end{proof}

\begin{lemma}\label{lemma: geometric lazy}
Let $\alpha \in [1, \infty)$. Then
\begin{align*}
    \sum_{i = 0}^{n}\left(1 + \frac{1}{\alpha}\right)^{i} \geq \left(1 + \frac{1}{\alpha}\right)^{n + 1} - \frac{1}{\alpha}
\end{align*}
\end{lemma}
\begin{proof}
Trivially, the statement is true for $n = 0$, because $(1 + 1/\alpha)^{0} = 1$ and $(1 + 1/\alpha)^{1} - 1/\alpha = 1$. Thus, we proceed by induction. Suppose
\begin{align*}
    \sum_{i = 0}^{n}\left(1 + \frac{1}{\alpha} \right)^{i} \geq \left(1 + \frac{1}{\alpha}\right)^{n + 1} - \frac{1}{\alpha}
\end{align*}
for some $n$. Then 
\begin{align*}
    \sum_{i = 0}^{n + 1}\left(1 + \frac{1}{\alpha}\right)^{i} &= \left(1 + \frac{1}{\alpha}\right)^{n + 1} + \sum_{i = 0}^{n}\left(1 + \frac{1}{\alpha}\right)^{i}\\
    &\geq \left(1 + \frac{1}{\alpha}\right)^{n + 1} + \left(1 + \frac{1}{\alpha}\right)^{n + 1} - \frac{1}{\alpha}\\
    &= 2\left(1 + \frac{1}{\alpha}\right)^{n + 1} - \frac{1}{\alpha}\\
    &\geq \left(1 + \frac{1}{\alpha}\right)\left(1 + \frac{1}{\alpha}\right)^{n + 1} - \frac{1}{\alpha} \tag{Since $2 \ge 1 + 1/\alpha$}\\
    &= \left(1 + \frac{1}{\alpha}\right)^{n + 2} - \frac{1}{\alpha}
\end{align*}
Thus, the claim holds for $n + 1$, which proves the lemma for all $n \geq 0$.
\end{proof}

\begin{lemma}\label{lemma: lazy lower order}
Algorithm \ref{alg: Lazy Algorithm} ensures that the summation of the counters across the episodes where we do not update have the following upper bounds:
\begin{align*}
    \sum_{s, a, h}\sum_{j = 1}^{J}\frac{n_{k_{j} + 1, h}^{k_{j + 1}}(s, a)}{N_{k_{j}h}'\left(s, a\right)}\leq \left(1 + \frac{1}{\alpha}\right)HSA + \left(1 + \frac{1}{\alpha}\right)HSA\log\left(\frac{K}{SA}\right) \leq 2\left(1 + \frac{1}{\alpha}\right)HSA\left(\frac{K}{SA}\right)
\end{align*}
where the final inequality holds for $K/SA \geq \exp(1)$.
\end{lemma}
\begin{proof}
To prove the result, we extend the summation to include the state-action-step triples in episode $k_j$ that did not trigger the update rule: 
\begin{align*}
    \sum_{s, a, h}\sum_{j = 1}^{J}\frac{n_{k_{j} + 1, h}^{k_{j + 1}}(s, a)}{N_{k_{j}h}'\left(s, a\right)} &\leq \sum_{s, a, h}\sum_{j = 1}^{J}\frac{n_{k_{j} h}^{k_{j + 1}}(s, a)}{N_{k_{j}h}'\left(s, a\right)} \mathds{1}\left\{n_{k_j h}(s, a) \leq \frac{1}{\alpha} N_{k_j h}'(s, a)\right\} \\
    &\leq \sum_{s, a, h}\left(\left(1 + \frac{1}{\alpha}\right) + \left(1 + \frac{1}{\alpha}\right)\log\left(N_{J}(s, a, h)\right)\right)\tag{Lemma \ref{lemma: lazy counters}}\\
    &= \left(1 + \frac{1}{\alpha}\right)HSA + \left(1 + \frac{1}{\alpha}\right)\sum_{s, a, h}\log\left(N_{J}(s, a, h)\right)\tag{Expand Summation}\\
    &\leq \left(1 + \frac{1}{\alpha}\right)HSA + \left(1 + \frac{1}{\alpha}\right)HSA\log\left(\frac{\sum_{s, a, h}N_{J}(s, a, h)}{HSA}\right)\tag{Jensen's}\\
    &\leq \left(1 + \frac{1}{\alpha}\right)HSA + \left(1 + \frac{1}{\alpha}\right)HSA\log\left(\frac{T}{HSA}\right)\tag{$\sum_{s, a, h}N_{J}(s, a, h) \leq T$}\\
    &= \left(1 + \frac{1}{\alpha}\right)HSA + \left(1 + \frac{1}{\alpha}\right)HSA\log\left(\frac{K}{SA}\right) \tag{$T = KH$}\\
    &\leq 2\left(1 + \frac{1}{\alpha}\right)HSA\log\left(\frac{K}{SA}\right)
\end{align*}
for $K/SA \geq \exp(1)$, as required.
\end{proof}

\subsection{Proof of Regret Bound for Lazy Updating}\label{sec: lazy regret bound proof}
\lazy*
\begin{tproof}
Let $\tilde{\Delta}_{h}^{k}(s) = \tilde{V}_{h}^{\pi_{k}}\left(s\right) - V_{h}^{\pi_{k}}\left(s\right)$ denote the difference between the optimistic and actual value of policy $\pi_{k}$ from state $s$ and step $h$. By definition, the regret of any episodic reinforcement learning algorithm is given by: 
\begin{align*}
    \mathfrak{R}_{K} &= \sum_{k = 1}^{K}\left( V_{1}^{*}\left(s_{1}^{k}\right) - V_{1}^{\pi_{k}}\left(s_{1}^{k}\right) \right)\\
    &\leq \sum_{k = 1}^{K}\left( \tilde{V}_{1}^{\pi_{k}}\left(s_{1}^{k}\right) - V_{1}^{\pi_{k}}\left(s_{1}^{k}\right) \right) = \sum_{k = 1}^{K}\tilde{\Delta}_{1}^{k}\left(s_{1}^{k}\right) = \sum_{j = 1}^{J}\sum_{k = k_{j}}^{k_{j + 1} - 1}\tilde{\Delta}_{1}^{k}\left(s_{1}^{k}\right) \\
    &= \sum_{j = 1}^{J}\sum_{k = k_{j}}^{k_{j + 1} - 1}\tilde{\Delta}_{1}^{k}\left(s_{1}^{k}\right)\mathds{1}\left\{k + \tau_{k} \geq k_{j + 1}\right\} = \sum_{j = 1}^{J}\tilde{\Delta}_{1}^{k_j}\left(s_{1}^{k_j}\right) + \sum_{j = 1}^{J}\sum_{k = k_{j}}^{k_{j + 1} - 1}\tilde{\Delta}_{1}^{k}\left(s_{1}^{k}\right) \\
    &= \underbrace{\sum_{j = 1}^{J}\tilde{\Delta}_{1}^{k_j}\left(s_{1}^{k_j}\right)}_{(i)} + 
    \underbrace{\sum_{j = 1}^{J}\sum_{k = k_{j} + 1}^{k_{j + 1} - 1}\tilde{\Delta}_{1}^{k}\left(s_{1}^{k}\right) \mathds{1}\left\{k + \tau_{k} \geq k_{j + 1}\right\} }_{(ii)} + 
    \underbrace{\sum_{j = 1}^{J}\sum_{k = k_{j} + 1}^{k_{j + 1} - 1}\tilde{\Delta}_{1}^{k}\left(s_{1}^{k}\right) \mathds{1}\left\{k + \tau_{k} < k_{j + 1}\right\} }_{(iii)}
\end{align*}
where the inequality follows from optimism, the penultimate equality follows from epochs consisting of disjoint sets of episodes and the final equality follows from splitting the episodes into three disjoint sets, $(i)$, $(ii)$, and $(ii)$: 
\begin{itemize}
    \item [$(i)$] episodes where we perform a policy update,
    \item [$(ii)$] episodes played in the $j$-th epoch but observed in epoch $j' > j$,
    \item [$(iii)$] episodes played in the $j$-th epoch and observed in the $j$-th epoch.
\end{itemize}

First, we focus on the episodes where we perform a policy update, e.g. $(i)$. Recall that Lemma \ref{lemma: doubling trick} tells us the total number of updates is logarithmic in the number of episodes. Further, the rewards are bounded between zero and one, meaning the regret of any episode is at most $H$. Combining these two results gives a trivial bound on regret of this term: $(i) \leq HJ$.\\

Next, we bound the regret of the episodes whose feedback is not observable before the start of the next epoch e.g., $(ii)$. Once again, we can rely on Lemma \ref{lemma: doubling trick} and the fact that the regret of any episode is at most $H$ to get a bound on this term that is logarithmic in $K$. Doing so gives the following result:
\begin{align*}
    (ii) &= \sum_{j = 1}^{J}\sum_{k = k_{j} + 1}^{k_{j + 1} - 1}\tilde{\Delta}_{1}^{k}\left(s_{1}^{k}\right)\mathds{1}\left\{k + \tau_{k} \geq k_{j + 1}\right\}\\
    &\leq H\sum_{j = 1}^{J}\sum_{k = k_{j} + 1}^{k_{j + 1} - 1}\mathds{1}\left\{k + \tau_{k} \geq k_{j + 1}\right\}\\
    &\leq H\sum_{j = 1}^{J}\sum_{k = 1}^{k_{j + 1} - 1}\mathds{1}\left\{k + \tau_{k} \geq k_{j + 1}\right\}\\
    &=  H\sum_{j = 1}^{J}S_{k_{j + 1}}\\
    &\leq HJ\psi_{K}^{\tau} \tag{$S_{k}\leq \psi_{k}^{\tau} \leq \psi_{K}^{\tau}$}
\end{align*}
Finally, we handle the episodes that are played and observed in the same epoch e.g., term $(iii)$. Lemma \ref{lemma: regret decomposition} allows us to make a start on bounding this term:

\begin{align*}
    (iii) &= \sum_{j = 1}^{J}\sum_{k = k_{j} + 1}^{k_{j + 1} - 1}\tilde{\Delta}_{1}^{k}\left(s_{1}^{k}\right)\mathds{1}\left\{k + \tau_{k} < k_{j + 1}\right\}\\
    &\leq 6\left(H + C\right)\sqrt{T\log\left(\frac{K\pi}{6\delta'}\right)} + 6\sum_{j = 1}^{J}\sum_{k = k_{j} + 1}^{k_{j + 1} - 1}\sum_{h = 1}^{H}\Bigg(\beta_{kh}\left(s_{h}^{k}, a_{h}^{k}\right) + \frac{3CH^2 SL}{N_{kh}'\left(s_{h}^{k}, a_{h}^{k}\right)}\Bigg)\mathds{1}\left\{k + \tau_{k} < k_{j + 1}\right\}\\
    &= 6\left(H + C\right)\sqrt{T\log\left(\frac{K\pi}{6\delta'}\right)}\\
    &\quad + 6\sum_{j = 1}^{J}\sum_{k = k_{j} + 1}^{k_{j + 1} - 1}\sum_{h = 1}^{H}\frac{3CH^2 SL}{N_{kh}'\left(s_{h}^{k}, a_{h}^{k}\right)}\mathds{1}\left\{k + \tau_{k} < k_{j + 1}\right\} \tag{$iii.a$}\\
    &\quad + 6\sum_{j = 1}^{J}\sum_{k = k_{j} + 1}^{k_{j + 1} - 1}\sum_{h = 1}^{H}\beta_{kh}\left(s_{h}^{k}, a_{h}^{k}\right)\mathds{1}\left\{k + \tau_{k} < k_{j + 1}\right\}\tag{$iii.b$}
\end{align*}

Thus, bounding $(iii)$ now amounts to finding an upper bounds for $(iii.a)$ and $(iii.b)$. Since $k_j$ does not feature in either summation, we know that 
$$
n_{k' h}^{k_{j + 1}}(s, a) \leq \frac{1}{\alpha}N_{k_j h}'(s, a)
$$
for all $(s, a, h)\in \mathcal{S}\times \mathcal{A} \times [H]$ and $k' \geq k_j + 1$. By introducing a summation over all the states-actions and steps, we can easily bound $(iii.a)$ via Lemma \ref{lemma: lazy lower order}:
\begin{align*}
    (iii.a) &= 3CH^2 SL\sum_{j = 1}^{J}\sum_{k = k_{j} + 1}^{k_{j + 1} - 1}\sum_{h = 1}^{H} \frac{\mathds{1}\left\{k + \tau_{k} < k_{j + 1}\right\}}{N_{k_{j}h}'\left(s_{h}^{k}, a_{h}^{k}\right)}\\
    &= 3CH^2 SL\sum_{s, a, h}\sum_{j = 1}^{J}\sum_{k = k_{j} + 1}^{k_{j + 1} - 1}\frac{\mathds{1}\left\{s_{h}^{k} = s, a_{h}^{k} = a, k + \tau_{k} < k_{j + 1}\right\}}{N_{k_{j}h}'\left(s, a\right)}\\
    &=  \left(B_{2} + 3CH^2 SL\right)\sum_{s, a, h}\sum_{j = 1}^{J}\frac{\sum_{k = k_{j} + 1}^{k_{j + 1} - 1}\mathds{1}\left\{s_{h}^{k} = s, a_{h}^{k} = a, k + \tau_{k} < k_{j + 1}\right\}}{N_{k_{j}h}'\left(s, a\right)}\\
    &= 3CH^2 SL\sum_{s, a, h}\sum_{j = 1}^{J}\frac{n_{k_{j} + 1, h}^{k_{j + 1}}(s, a)}{N_{k_{j}h}'\left(s, a\right)} \tag{Eq. \eqref{eqn: within counter}}\\
    &\leq 3CH^2 SL \left(2\left(1 + \frac{1}{\alpha}\right)HSA\log\left(\frac{K}{SA}\right)\right)\tag{By Lemma \ref{lemma: lazy lower order}}\\
    &= 6\left(1 + \frac{1}{\alpha}\right)CH^3 S^2 AL \log\left(\frac{K}{SA}\right)
\end{align*}
Bounding $(iii.b)$ requires some care due to the various forms of $B_1$ e.g., those that remain constant and those that utilise variance reduction techniques. By Lemma \ref{lemma: lazy counters}, it is clear that the summation of the visitation counters no longer depends on the delay. Therefore, we begin by an application of Cauchy-Schwarz (CS) to separate the numerator of the exploration bonus from the summation of the visitation counters:

\begin{align*}
    (iii.b) &= \sum_{j = 1}^{J}\sum_{k = k_{j} + 1}^{k_{j + 1} - 1}\sum_{h = 1}^{H}\Bigg(  \frac{B_{1}}{\sqrt{N_{k_{j}h}'\left(s_{h}^{k}, a_{h}^{k}\right)}} + \frac{B_2}{N_{k_{j}h}'\left(s_{h}^{k}, a_{h}^{k}\right)}\Bigg)\mathds{1}\left\{k + \tau_{k} < k_{j + 1}\right\}\\
    &= \sum_{j = 1}^{J}\sum_{k = k_{j}+1}^{k_{j + 1} - 1}\sum_{h = 1}^{H} \frac{B_{1} \mathds{1}\left\{k + \tau_{k} < k_{j + 1}\right\} }{\sqrt{N_{k_{j}h}'\left(s_{h}^{k}, a_{h}^{k}\right)}} + B_2 \sum_{j = 1}^{J}\sum_{k = k_{j}+1}^{k_{j + 1} - 1}\sum_{h = 1}^{H} \frac{\mathds{1}\left\{k + \tau_{k} < k_{j + 1}\right\} }{N_{k_{j}h}'\left(s_{h}^{k}, a_{h}^{k}\right)}\\
    &\leq \sqrt{ \sum_{j = 1}^{J}\sum_{k = k_{j}+1}^{k_{j + 1} - 1}\sum_{h = 1}^{H} B_{1}^{2}\,\mathds{1}\left\{k + \tau_{k} < k_{j + 1}\right\}\, \sum_{j = 1}^{J}\sum_{k = k_{j}+1}^{k_{j + 1} - 1}\sum_{h = 1}^{H} \frac{\mathds{1}\left\{k + \tau_{k} < k_{j + 1}\right\}}{N_{k_{j}h}'\left(s_{h}^{k}, a_{h}^{k}\right)}} \\
    & + B_2 \sum_{j = 1}^{J}\sum_{k = k_{j}+1}^{k_{j + 1} - 1}\sum_{h = 1}^{H} \frac{\mathds{1}\left\{k + \tau_{k} < k_{j + 1}\right\} }{N_{k_{j}h}'\left(s_{h}^{k}, a_{h}^{k}\right)} \\
    &= \sqrt{ \sum_{j = 1}^{J}\sum_{k = k_{j}+1}^{k_{j + 1} - 1}\sum_{h = 1}^{H} B_{1}^{2} \, \mathds{1}\left\{k + \tau_{k} < k_{j + 1}\right\}\, \sum_{s, a, h}\sum_{j = 1}^{J}\sum_{k = k_{j} +1}^{k_{j + 1} - 1} \frac{\mathds{1}\left\{s_{h}^{k} = s, a_{h}^{k} = a, k + \tau_{k} < k_{j + 1}\right\}}{N_{k_{j}h}'\left(s_{h}^{k}, a_{h}^{k}\right)}}\\
    &\quad + B_2 \sum_{s, a, h}\sum_{j = 1}^{J}\sum_{k = k_{j}+1}^{k_{j + 1} - 1} \frac{\mathds{1}\left\{s_{h}^{k} = s, a_{h}^{k} = a, k + \tau_{k} < k_{j + 1}\right\}}{N_{k_{j}h}'\left(s_{h}^{k}, a_{h}^{k}\right)}\\
    &= \sqrt{ \sum_{j = 1}^{J}\sum_{k = k_{j}+1}^{k_{j + 1} - 1}\sum_{h = 1}^{H} B_{1}^{2}\, \mathds{1}\left\{k + \tau_{k} < k_{j + 1}\right\}\,\sum_{s, a, h}\sum_{j = 1}^{J} \frac{\sum_{k = k_{j}+1}^{k_{j + 1} - 1} \mathds{1}\left\{s_{h}^{k} = s, a_{h}^{k} = a, k + \tau_{k} < k_{j + 1}\right\}}{N_{k_{j}h}'\left(s_{h}^{k}, a_{h}^{k}\right)}}\\
    &\quad + B_2 \sum_{s, a, h}\sum_{j = 1}^{J} \frac{\sum_{k = k_{j}+1}^{k_{j + 1} - 1}\mathds{1}\left\{s_{h}^{k} = s, a_{h}^{k} = a, k + \tau_{k} < k_{j + 1}\right\}}{N_{k_{j}h}'\left(s_{h}^{k}, a_{h}^{k}\right)}\\
    &= \sqrt{ \sum_{j = 1}^{J}\sum_{k = k_{j}+1}^{k_{j + 1} - 1}\sum_{h = 1}^{H} B_{1}^{2}\, \mathds{1}\left\{k + \tau_{k} < k_{j + 1}\right\}\,\sum_{s, a, h}\sum_{j = 1}^{J} \frac{n_{k_{j} +1, h}^{k_{j + 1}}(s, a)}{N_{k_{j}h}'\left(s_{h}^{k}, a_{h}^{k}\right)}} + B_2 \sum_{s, a, h}\sum_{j = 1}^{J} \frac{n_{k_{j} +1, h}^{k_{j + 1}}(s, a)}{N_{k_{j}h}'\left(s_{h}^{k}, a_{h}^{k}\right)} \tag{Eq. \eqref{eqn: within counter}}\\
    &\leq \sqrt{2\left(1 + \frac{1}{\alpha}\right)HSA\log\left(\frac{K}{SA}\right) \sum_{j = 1}^{J}\sum_{k = k_{j}+1}^{k_{j + 1} - 1}\sum_{h = 1}^{H} B_{1}^{2}\,\mathds{1}\left\{k + \tau_{k} < k_{j + 1}\right\}} \\
    &\quad + 2\left(1 + \frac{1}{\alpha}\right)B_2 HSA\log\left(\frac{K}{SA}\right)\tag{Lemma \ref{lemma: lazy lower order}}\\
    &\leq \left(1 + \frac{1}{\alpha}\right)\hat{\mathfrak{R}}_{K}(Base) \log\left(\frac{K}{SA}\right)\\
    &\lesssim\left(1 + \frac{1}{\alpha}\right)\hat{\mathfrak{R}}_{K}(Base)
\end{align*}

The penultimate line in the above is simply the sum of the bonuses for the chosen base algorithm under immediate feedback scaled by a logarithmic factor, which is introduced by the slower updating. For $B_{1} \approx B$ e.g., the upper bound only involves inflating terms inside logarithms, one can upper bound the summation under the square-root by $T B^{2}$, which is tight up to logarithmic factors. When $B_{1}$ involves some form of empirical variance term, one can use the techniques outlined by \cite{PB2020, UCBVI, UCRL2B} to bound the summation under the square-root by $\approx HT$; once again this too is tight up to logarithmic factors. More simply, the epochs form a simulated non-delayed version of the environment for the base algorithm. Therefore, $(iii.b)$ can be replaced with the upper bound of the regret in the non-delayed environment multiplied by the extra logarithmic factors that arise from the slower updating, because the summation of the bonuses are the leading term in the regret bound.

Bringing everything together gives: 
\begin{align*}
    \mathfrak{R}_{K} &\leq (i) + (ii) + (iii.a) + (iii.b)\\
    &\lesssim \left(1 + \frac{1}{\alpha}\right) \hat{\mathfrak{R}}_{K}(Base) + HJ\psi_{K}^{\tau}\\
    &\lesssim \left(1 + \frac{1}{\alpha}\right)\hat{\mathfrak{R}}_{K}(Base) + \frac{ H^2 SA \psi_{K}^{\tau}}{\log(1 + \frac{1}{\alpha})}\\
\end{align*}
Plugging in $\psi_{K}^{\tau}$ (and suppressing poly-logarithmic factors) gives the stated result.
\end{tproof}

\section{Additional Theoretical Results}

Here, we present a brief overview of the results that unify model-optimistic and value-optimistic model-based episodic reinforcement learning algorithms \citep{PB2020}. The class of model-optimistic algorithms explicitly define the following failure event for some divergence $D(\hat{P}_kh(\cdot\vert s, a), P_h(\cdot\vert s, a)$:
\begin{align*}
    F_{k}^{p} &= \left\{\exists\, s, a, h: D\left(\hat{P}_{kh}\left(\cdot\vert s, a\right), P_{h}\left(\cdot\vert s, a\right) \right) \geq \epsilon_{kh}^{p}\left(s, a\right)\right\}
\end{align*}
which holds across all episodes with probability $\delta'$. Indeed, $D$ must satisfy some conditions. Namely, $D$ must be jointly convex in its arguments so that $\mathcal{P}_{kh}$ (defined below) is convex, and it must be positive homogeneous.\footnote{The distance $\norm{p - p'}$ for any norm and all f -divergences satisfy these conditions \citep{Liese2006}.} Outside the failure event, with probability $1 - \delta'$, the  divergence between the empirical and actual transition density of the $h^\text{th}$ step at the start of the $k^\text{th}$ episode is therefore, at most: $D(\hat{P}_{kh}(\cdot\vert s, a), P_{h}(\cdot\vert s, a)) \leq \epsilon_{kh}^{p}(s, a)$. Using $\epsilon_{kh}^{p}(s, a)$ as the maximum divergence allows for the construction of the following plausible set: 
\begin{align*}
    \mathcal{P}_{kh} = \left\{\tilde{P}_{h}\left(\cdot\vert s, a\right) \in \Delta: D\left(\tilde{P}_{h}\left(\cdot\vert s, a\right), \hat{P}_{kh}\left(\cdot\vert s, a\right)\right) \leq \epsilon_{kh}^{p}\left(s, a\right)\right\}
\end{align*}
for each $(s, a, h)\in \mathcal{S}\times\mathcal{A}\times[H]$. Here, $\Delta$ denotes the set of valid transition densities. From here, it is possible to derive the bonus by finding the conjugate of the divergence: 
\begin{align*}
    \beta_{kh}^{*}\left(s, a\right) &= \max_{\tilde{P}_{h}\left(\cdot\vert s, a\right) \in \Delta}\left\{\Big\langle \tilde{V}, \tilde{P}_{h}\left(\cdot\vert s, a\right) - \hat{P}_{h}\left(\cdot\vert s, a\right)\Big\rangle\right\}\\
    \beta_{kh}^{-}\left(s, a\right) &= \max_{\tilde{P}_{h}\left(\cdot\vert s, a\right) \in \Delta}\left\{\Big\langle -\tilde{V}, \tilde{P}_{h}\left(\cdot\vert s, a\right) - \hat{P}_{h}\left(\cdot\vert s, a\right)\Big\rangle\right\}\\
    \beta_{kh}\left(s, a\right) &\geq \max\left\{\beta_{kh}^{*}\left(s, a\right), \beta_{kh}^{-}\left(s, a\right)\right\}
\end{align*}
by introducing a Lagrange multiplier. For a derivation of the bonuses associated with each divergence, we refer the reader to Appendix A.5 of \citet{PB2020}.

\subsection{Missing Proofs for the Regret Decomposition}\label{sec: missing proof for decomposition}
In this subsection, we utilise the fact that all model-based algorithms compute an optimistic value function of the form \eqref{eqn: optimistic value function} to derive an adaptable regret decomposition. The decomposition is adaptable in the sense it allows for tighter delay-dependence when the bonuses satisfy a symmetry-like property.\\

Throughout, we assume that the model-based algorithm is optimistic with high probability. That is, $\tilde{V}_{h}^{\pi_{k}}(s) \geq V_{h}^{*}(s) \geq V_{h}^{\pi_{k}}(s)$ with high probability at least $1 - \delta'$. Further, $C$ is defined as the event where: 

$$
\beta_{kh}^{+}(s, a)) \geq \Big\langle \left(\hat{P}_{kh} - P_{h}\right)\left(\cdot \,\vert s, a \right), \tilde{V}_{h + 1}^{\pi_{k}}(\cdot) \Big\rangle
$$
which holds across all episodes for every state-action-step triple conditional on the complement of the failure event.

\decomposition*
\begin{proof}
By definition, the regret of any episodic reinforcement learning algorithm is given by: 
\begin{align*}
    \mathfrak{R}_{K} &= \sum_{k = 1}^{K}V_{1}^{*}\left(s_{1}^{k}\right) - V_{1}^{\pi_{k}}\left(s_{1}^{k}\right) \leq \sum_{k = 1}^{K}\tilde{V}_{1}^{\pi_{k}}\left(s_{1}^{k}\right) - V_{1}^{\pi_{k}}\left(s_{1}^{k}\right)
\end{align*}
where the final inequality holds by optimism, which holds across all episodes with probability at least $1 - \delta'$. Consider the more general case of bounding the regret from the $h$-th step of each episode, rather than just the first step. Define $\tilde{\Delta}_{h}^{k} = \tilde{V}_{h}^{\pi_{k}}(s_{h}^{k}) - V_{h}^{\pi_{k}}(s_{h}^{k})$. Applying Lemma \ref{lemma: recursive decomposition} gives, with probability at least $1 - \delta'$:
\begin{align*}
\tilde{\Delta}_{h}^{k}(s_{h}^{k}) 
    &\leq \left(1 + \frac{C}{H}\right)\tilde{\Delta}_{h + 1}^{k}\left(s_{h + 1}^{k}\right) + 2\beta_{kh}\left(s_{h}^{k}, a_{h}^{k}\right) + \frac{6CH^2 SL}{N_{kh}'\left(s_{h}^{k}, a_{h}^{k}\right)} + \zeta_{h + 1}^{k} + C\bar{\zeta}_{h + 1}^{k}
\end{align*}
where 
\begin{align*}
    \zeta_{h + 1}^{k} & \coloneqq \big\langle P_{h}\left(\cdot\,\vert\,s_{h}^{k}, a_{h}^{k}\right),\, \tilde{\Delta}_{h + 1}^{k}\left(\cdot\right)\big\rangle  - \tilde{\Delta}_{h + 1}^{k}\left(s_{h + 1}^{k}\right)\\
    \bar{\zeta}_{h + 1}^{k}&\coloneqq\sqrt{\frac{4L}{N_{kh}'\left(s_{h}^{k}, a_{h}^{k}\right)}}\left[\left(\sum_{s'\in G_{kh}}P_{h}\left(s'\,\vert\,s_{h}^{k}, a_{h}^{k}\right)\frac{\tilde{\Delta}_{h + 1}^{k}\left(s'\right)}{\sqrt{P_{h}\left(s'\,\vert\,s_{h}^{k}, a_{h}^{k}\right)}}\right) - \frac{\tilde{\Delta}_{h + 1}^{k}\left(s_{h + 1}^{k}\right)}{\sqrt{P_{h}\left(s_{h + 1}^{k}\,\vert\,s_{h}^{k}, a_{h}^{k}\right)}}\right]
\end{align*}
and 
\begin{align*}
G_{kh} &\coloneqq \{s': P_{h}\left(s'\,\vert\,s_{h}^{k}, a_{h}^{k}\right) N_{kh}'\left(s_{h}^{k}, a_{h}^{k}\right) \geq 4H^2 L\}
\end{align*}
Now, we can utilise the recursive decomposition above to show that:    
\begin{align*}
    \tilde{\Delta}_{j}^{k}\left(s_{j}^{k}\right) \leq \left(1 + \frac{C}{H}\right)^{H - j}\sum_{h = j}^{H} 2\beta_{kh}\left(s_{h}^{k}, a_{h}^{k}\right) + \frac{6CH^2 SL}{N_{kh}'\left(s_{h}^{k}, a_{h}^{k}\right)} + \zeta_{h + 1}^{k} + C\bar{\zeta}_{h + 1}^{k}
\end{align*}
which we do by induction.  Recall that: $\tilde{V}_{H + 1}^{\pi_{k}} = V_{H + 1}^{*} = V_{H + 1}^{\pi_{k}} = \vec{0}$. Therefore, the statement holds when $j = H$, because: $\tilde{\Delta}_{H + 1}^{k} =  0$. Now assume the statement holds for $h = j + 1$. Then,  
\begin{align*}
    \tilde{\Delta}_{j}^{k}\left(s_{j}^{k}\right) &\leq \left(1 + \frac{C}{H}\right)\tilde{\Delta}_{j + 1}^{k}\left(s_{j + 1}^{k}\right) + \left(2\beta_{kh}\left(s_{h}^{k}, a_{h}^{k}\right) + \frac{6CH^2 SL}{N_{kh}'\left(s_{h}^{k}, a_{h}^{k}\right)} + \zeta_{h + 1}^{k} + C\bar{\zeta}_{h + 1}^{k}\right)\\
    &\leq \left(1 + \frac{C}{H}\right)\left(\left(1 + \frac{C}{H}\right)^{H - (j + 1)} \sum_{h = j + 1}^{H} 2\beta_{kh}\left(s_{h}^{k}, a_{h}^{k}\right) + \frac{6CH^2 SL}{N_{kh}'\left(s_{h}^{k}, a_{h}^{k}\right)} + \zeta_{h + 1}^{k} + C\bar{\zeta}_{h + 1}^{k}\right)\\
    &\quad+ \left(2\beta_{kh}\left(s_{h}^{k}, a_{h}^{k}\right) + \frac{6CH^2 SL}{N_{kh}'\left(s_{h}^{k}, a_{h}^{k}\right)} + \zeta_{h + 1}^{k} + C\bar{\zeta}_{h + 1}^{k}\right)\\
    &= \left(1 + \frac{C}{H}\right)^{H - j} \sum_{h = j + 1}^{H} \left(2\beta_{kh}\left(s_{h}^{k}, a_{h}^{k}\right) + \frac{6CH^2 SL}{N_{kh}'\left(s_{h}^{k}, a_{h}^{k}\right)} + \zeta_{h + 1}^{k} + C\bar{\zeta}_{h + 1}^{k}\right)\\
    &\quad+ \left(2\beta_{kh}\left(s_{h}^{k}, a_{h}^{k}\right) + \frac{6CH^2 SL}{N_{kh}'\left(s_{h}^{k}, a_{h}^{k}\right)} + \zeta_{h + 1}^{k} + C\bar{\zeta}_{h + 1}^{k}\right)\\
    &\leq \left(1 + \frac{C}{H}\right)^{H - j} \sum_{h = j + 1}^{H} \left(2\beta_{kh}\left(s_{h}^{k}, a_{h}^{k}\right) + \frac{6CH^2 SL}{N_{kh}'\left(s_{h}^{k}, a_{h}^{k}\right)} + \zeta_{h + 1}^{k} + C\bar{\zeta}_{h + 1}^{k}\right)\\
    &\quad+ \left(1 + \frac{C}{H}\right)^{H - j} \left(2\beta_{kh}\left(s_{h}^{k}, a_{h}^{k}\right) + \frac{6CH^2 SL}{N_{kh}'\left(s_{h}^{k}, a_{h}^{k}\right)} + \zeta_{h + 1}^{k} + C\bar{\zeta}_{h + 1}^{k}\right)\\
    &\leq \left(1 + \frac{C}{H}\right)^{H - j} \sum_{h = j}^{H} \left(2\beta_{kh}\left(s_{h}^{k}, a_{h}^{k}\right) + \frac{6CH^2 SL}{N_{kh}'\left(s_{h}^{k}, a_{h}^{k}\right)} + \zeta_{h + 1}^{k} + C\bar{\zeta}_{h + 1}^{k}\right)
\end{align*}

Therefore, we are now able to upper bound the regret as follows: 
\begin{align*}
   \mathfrak{R}_{K}  &\leq \sum_{k = 1}^{K} \tilde{\Delta}_{1}^{k}\left(s_{1}^{k}\right)\\
   &\leq \underbrace{\left(1 + \frac{C}{H}\right)^{H}}_{\leq \,e < 3}\left(C\sum_{k = 1}^{K}\sum_{h = 1}^{H}\bar{\zeta}_{h + 1}^{k} + \sum_{k = 1}^{K}\sum_{h = 1}^{H}\zeta_{h + 1}^{k} + 2\sum_{k = 1}^{K}\sum_{h = 1}^{H}\beta_{kh}\left(s_{h}^{k}, a_{h}^{k}\right) + \sum_{k = 1}^{K}\sum_{h = 1}^{H}\frac{6CH^2 SL}{N_{kh}'\left(s_{h}^{k}, a_{h}^{k}\right)}\right)\\
   &\leq 3C\sum_{k = 1}^{K}\sum_{h = 1}^{H}\bar{\zeta}_{h + 1}^{k} + 3\sum_{k = 1}^{K}\sum_{h = 1}^{H}\zeta_{h + 1}^{k} + 6\sum_{k = 1}^{K}\sum_{h = 1}^{H}\beta_{kh}\left(s_{h}^{k}, a_{h}^{k}\right) + 
   \sum_{k = 1}^{K}\sum_{h = 1}^{H}\frac{18CH^2 SL}{N_{kh}'\left(s_{h}^{k}, a_{h}^{k}\right)}\\
   &= 3C\sum_{k = 1}^{K}\sum_{h = 1}^{H}\bar{\zeta}_{h + 1}^{k} + 3\sum_{k = 1}^{K}\sum_{h = 1}^{H}\zeta_{h + 1}^{k} + 6\sum_{k = 1}^{K}\sum_{h = 1}^{H}\left(\beta_{kh}\left(s_{h}^{k}, a_{h}^{k}\right) + \frac{3CH^2 SL}{N_{kh}'\left(s_{h}^{k}, a_{h}^{k}\right)}\right)
\end{align*}
Recall the definitions of $\zeta_{h + 1}^{k}$ and $\bar{\zeta}_{h + 1}^{k}$:
\begin{align*}
    \zeta_{h + 1}^{k} &= \Big\langle P_{h}\left(\cdot\,\vert\,s_{h}^{k}, a_{h}^{k}\right), \tilde{\Delta}_{h + 1}^{k}\left(\cdot\right) \Big\rangle - \tilde{\Delta}_{h + 1}^{k}(s_{h + 1}^{k})\\
    \bar{\zeta}_{h + 1}^{k} &= \sqrt{\frac{4L}{N_{kh}'\left(s_{h}^{k}, a_{h}^{k}\right)}}\left[\left(\sum_{s'\in G_{kh}}\frac{P_{h}\left(s'\vert s_{h}^{k}, a_{h}^{k}\right)\tilde{\Delta}_{h + 1}^{k}\left(s'\right)}{\sqrt{P_{h}\left(s'\vert s_{h}^{k}, a_{h}^{k}\right)}}\right) - \frac{\Delta_{h + 1}^{k}\left(s_{h + 1}^{k}\right)}{\sqrt{P_{h}\left(s_{h + 1}^{k}\vert s_{h}^{k}, a_{h}^{k}\right)}}\right]
\end{align*}
with 
\begin{align*}
G_{kh} &\coloneqq \{s': P_{h}\left(s'\,\vert\,s_{h}^{k}, a_{h}^{k}\right) N_{kh}'\left(s_{h}^{k}, a_{h}^{k}\right) \geq 4H^2 L\}
\end{align*}
Let $\mathcal{F}_{kh} = \sigma(\{\mathcal{H}_{i}\}_{i: i + \tau_{i} < k})$ be the natural filtration of the observed information. Then $\abs{\zeta_{h + 1}^{k}} \leq 2H$ and
\begin{align*}
    &\mathbb{E}_{s_{h + 1}^{k}\sim P_{h}(\cdot\,\vert\,s_{h}^{k}, a_{h}^{k})}\left[\zeta_{h + 1}^{k}\,\vert\,\mathcal{F}_{kh}\cup \{s_{h}^{k}, a_{h}^{k}\}\right]\\
    &= \mathbb{E}_{s_{h + 1}^{k}\sim P_{h}(\cdot\,\vert\,s_{h}^{k}, a_{h}^{k})}\left[ \Big\langle P_{h}\left(\cdot\,\vert\,s_{h}^{k}, a_{h}^{k}\right), \tilde{\Delta}_{h + 1}^{k}\left(\cdot\right) \Big\rangle -  \tilde{\Delta}_{h + 1}^{k}\left(s_{h + 1}\right)\,\Big\vert\,\mathcal{F}_{kh}\cup \{s_{h}^{k}, a_{h}^{k}\}\right] \\
    &= \Big\langle P_{h}\left(\cdot\,\vert\,s_{h}^{k}, a_{h}^{k}\right), \tilde{\Delta}_{h + 1}^{k}\left(\cdot\right) \Big\rangle - \mathbb{E}_{s_{h + 1}^{k}\sim P_{h}(\cdot\,\vert\,s_{h}^{k}, a_{h}^{k})}\left[\tilde{\Delta}_{h + 1}^{k}\left(s_{h + 1}^{k}\right)\,\Big\vert\,\mathcal{F}_{kh}\cup \{s_{h}^{k}, a_{h}^{k}\}\right]  \\
    &= \Big\langle P_{h}\left(\cdot\,\vert\,s_{h}^{k}, a_{h}^{k}\right), \tilde{\Delta}_{h + 1}^{k}\left(\cdot\right) \Big\rangle  - \Big\langle P_{h}\left(\cdot\,\vert\,s_{h}^{k}, a_{h}^{k}\right), \tilde{\Delta}_{h + 1}^{k}\left(\cdot\right) \Big\rangle  = 0
\end{align*}
Similarly, $\abs{\bar{\zeta}_{h + 1}^{k}} \leq 2$ and $\mathbb{E}_{s_{h + 1}^{k}\sim P_{h}(\cdot\,\vert\,s_{h}^{k}, a_{h}^{k})}\left[\bar{\zeta}_{h + 1}^{k}\,\vert\,\mathcal{F}_{kh}\cup \{s_{h}^{k}, a_{h}^{k}\}\,, s_{h + 1}^{k}\in G_{kh}\right] = 0$. Therefore, $\zeta_{h + 1}^{k}$ and $\bar{\zeta}_{h + 1}^{k}$ are martingale differences, which are easily bounded using Azuma-Hoeffding:
\begin{align*}
    \sum_{k = 1}^{K}\sum_{h = 1}^{H}\zeta_{h + 1}^{k} &\leq 2H\sqrt{T\log\left(\frac{K\pi}{6\delta'}\right)}\tag{with probability at least $1 - \delta'$}\\
    \sum_{k = 1}^{K}\sum_{h = 1}^{H} \bar{\zeta}_{h + 1}^{k} &\leq 2\sqrt{T\log\left(\frac{K\pi}{6\delta'}\right)}\tag{with probability at least $1 - \delta'$}
\end{align*}

Therefore, with probability $1 - 4\delta'$: 
\begin{align*}
    \mathfrak{R}_{K} &\leq 6C\sqrt{T\log\left(\frac{K\pi}{6\delta'}\right)} + 6H\sqrt{T\log\left(\frac{K\pi}{6\delta'}\right)} + 6\sum_{k = 1}^{K}\sum_{h = 1}^{H}\left(\beta_{kh}\left(s_{h}^{k}, a_{h}^{k}\right) + \frac{3CH^2 SL}{N_{kh}'\left(s_{h}^{k}, a_{h}^{k}\right)}\right)\\
    &\leq 6(H + C)\sqrt{T\log\left(\frac{K\pi}{6\delta'}\right)} + 6\sum_{k = 1}^{K}\sum_{h = 1}^{H}\left(\beta_{kh}\left(s_{h}^{k}, a_{h}^{k}\right) + \frac{3CH^2 SL}{N_{kh}'\left(s_{h}^{k}, a_{h}^{k}\right)}\right)
\end{align*}
as required. 
\end{proof}

\begin{lemma}\label{lemma: recursive decomposition}
Let $C$ be an algorithm dependent-constant indicating whether it is model-optimistic or value-optimistic. Under Assumption \ref{assumption: estimation error bonus}, the regret of any optimistic model-based algorithm from the $h$-th step of the $k$-th episode upper bounded by: 
\begin{align*}
\tilde{\Delta}_{h}^{k}(s_{h}^{k}) 
    &\leq \left(1 + \frac{C}{H}\right)\tilde{\Delta}_{h + 1}^{k}\left(s_{h + 1}^{k}\right) + 2\beta_{kh}\left(s_{h}^{k}, a_{h}^{k}\right) + \frac{6CH^2 SL}{N_{kh}'\left(s_{h}^{k}, a_{h}^{k}\right)}\\
    &+ \big\langle P_{h}\left(\cdot\,\vert\,s_{h}^{k}, a_{h}^{k}\right),\, \tilde{\Delta}_{h + 1}^{k}\left(\cdot\right)\big\rangle  - \tilde{\Delta}_{h + 1}^{k}\left(s_{h + 1}^{k}\right)\\
    &+ \sqrt{\frac{4CL}{N_{kh}'\left(s_{h}^{k}, a_{h}^{k}\right)}}\left[\left(\sum_{s'\in G_{kh}}P_{h}\left(s'\,\vert\,s_{h}^{k}, a_{h}^{k}\right)\frac{\tilde{\Delta}_{h + 1}^{k}\left(s'\right)}{\sqrt{P_{h}\left(s'\,\vert\,s_{h}^{k}, a_{h}^{k}\right)}}\right) - \frac{\tilde{\Delta}_{h + 1}^{k}\left(s_{h + 1}^{k}\right)}{\sqrt{P_{h}\left(s_{h + 1}^{k}\,\vert\,s_{h}^{k}, a_{h}^{k}\right) }}\right]
\end{align*}
where $L = \log(S^2 AH \pi^2 / 6\delta')$ and
\begin{align*}
G_{kh} &\coloneqq \{s': P_{h}\left(s'\,\vert\,s_{h}^{k}, a_{h}^{k}\right) N_{kh}'\left(s_{h}^{k}, a_{h}^{k}\right) \geq 4H^2 L\}
\end{align*}
with probability $1 - \delta'$.
\end{lemma}
\begin{proof}
By Proposition 2 of \citet{PB2020} and by definition of the value-optimistic algorithms, we have that: 
\begin{align*}
    \tilde{\Delta}_{h}^{k}(s_{h}^{k}) &= \tilde{V}_{h}^{\pi_{k}}\left(s_{h}^{k}\right) - V_{h}^{\pi_{k}}\left(s_{h}^{k}\right) = \beta_{kh}^{+}\left(s_{h}^{k}, a_{h}^{k} \right) + \big\langle\hat{P}_{kh}\left(\cdot\,\vert\, s_{h}^{k}, a_{h}^{k}\right) \tilde{V}_{h + 1}^{\pi_{k}}\big\rangle- \big\langle P_{h}\left(\cdot\,\vert\, s_{h}^{k}, a_{h}^{k}\right) V_{h + 1}^{\pi_{k}} \big\rangle\\
    &= \beta_{kh}^{+}\left(s_{h}^{k}, a_{h}^{k} \right) + \big\langle\hat{P}_{kh}\left(\cdot\,\vert\, s_{h}^{k}, a_{h}^{k}\right) - P_{h}\left(\cdot\,\vert\, s_{h}^{k}, a_{h}^{k}\right), \tilde{V}_{h + 1}^{\pi_{k}}\big\rangle +  \big\langle P_{h}\left(\cdot\,\vert\, s_{h}^{k}, a_{h}^{k}\right), \tilde{V}_{h + 1}^{\pi_{k}} - V_{h + 1}^{\pi_{k}} \big\rangle\\
    &= \beta_{kh}^{+}\left(s_{h}^{k}, a_{h}^{k} \right) + \big\langle\hat{P}_{kh}\left(\cdot\,\vert\, s_{h}^{k}, a_{h}^{k}\right) - P_{h}\left(\cdot\,\vert\, s_{h}^{k}, a_{h}^{k}\right), \tilde{V}_{h + 1}^{\pi_{k}}\big\rangle +  \big\langle P_{h}\left(\cdot\,\vert\, s_{h}^{k}, a_{h}^{k}\right), \tilde{\Delta}_{h + 1}^{k} \big\rangle\\
    &\leq \beta_{kh}\left(s_{h}^{k}, a_{h}^{k} \right) + \big\langle\hat{P}_{kh}\left(\cdot\,\vert\, s_{h}^{k}, a_{h}^{k}\right) - P_{h}\left(\cdot\,\vert\, s_{h}^{k}, a_{h}^{k}\right), \tilde{V}_{h + 1}^{\pi_{k}}\big\rangle +  \big\langle P_{h}\left(\cdot\,\vert\, s_{h}^{k}, a_{h}^{k}\right), \tilde{\Delta}_{h + 1}^{k}\big\rangle\\
    &= \tilde{\Delta}_{h + 1}^{k}\left(s_{h + 1}^{k}\right) + \beta_{kh}\left(s_{h}^{k}, a_{h}^{k} \right)\\
    &\quad + \big\langle\hat{P}_{kh}\left(\cdot\,\vert\, s_{h}^{k}, a_{h}^{k}\right) - P_{h}\left(\cdot\,\vert\, s_{h}^{k}, a_{h}^{k}\right), \tilde{V}_{h + 1}^{\pi_{k}}\big\rangle \\
    &\quad+  \big\langle P_{h}\left(\cdot\,\vert\, s_{h}^{k}, a_{h}^{k}\right), \tilde{\Delta}_{h + 1}^{k}\big\rangle - \tilde{\Delta}_{h + 1}^{k}\left(s_{h + 1}^{k}\right)
\end{align*}
where the inequality follows from the fact that $\beta_{kh}(s, a)^{+} \leq \beta_{kh}(s, a)$. For model-optimistic algorithms, from the definition of the bonuses, we have that: 
\begin{align*}
    \big\langle\hat{P}_{kh}\left(\cdot\,\vert\, s_{h}^{k}, a_{h}^{k}\right) - P_{h}\left(\cdot\,\vert\, s_{h}^{k}, a_{h}^{k}\right), \tilde{V}_{h + 1}^{\pi_{k}}\big\rangle \leq \beta_{kh}^{-} \left(s_{h}^{k}, a_{h}^{k}\right) \leq \beta_{kh}\left(s_{h}^{k}, a_{h}^{k}\right)
\end{align*}
However, this term cannot be bound as easily for the value-optimistic algorithms. But, Assumption \ref{assumption: estimation error bonus} allows us to show that, with probability $1 - \delta'$: 
\begin{align*}
    &\big\langle\hat{P}_{kh}\left(\cdot\,\vert\, s_{h}^{k}, a_{h}^{k}\right) - P_{h}\left(\cdot\,\vert\, s_{h}^{k}, a_{h}^{k}\right), \tilde{V}_{h + 1}^{\pi_{k}}\big\rangle \\
    &=  \big\langle\hat{P}_{kh}\left(\cdot\,\vert\, s_{h}^{k}, a_{h}^{k}\right) - P_{h}\left(\cdot\,\vert\, s_{h}^{k}, a_{h}^{k}\right), V_{h + 1}^{*}\big\rangle
    +  \big\langle\hat{P}_{kh}\left(\cdot\,\vert\, s_{h}^{k}, a_{h}^{k}\right) - P_{h}\left(\cdot\,\vert\, s_{h}^{k}, a_{h}^{k}\right), \tilde{V}_{h + 1}^{\pi_{k}} - V_{h + 1}^{*}\big\rangle\\
    &\leq \beta_{kh}\left(s_{h}^{k}, a_{h}^{k}\right) + \big\langle\hat{P}_{kh}\left(\cdot\,\vert\, s_{h}^{k}, a_{h}^{k}\right) - P_{h}\left(\cdot\,\vert\, s_{h}^{k}, a_{h}^{k}\right), \tilde{V}_{h + 1}^{\pi_{k}} - V_{h + 1}^{*}\big\rangle \tag{By Assumption \ref{assumption: estimation error bonus}}\\
    &\leq \beta_{kh}\left(s_{h}^{k}, a_{h}^{k}\right) + \big\langle\abs{\hat{P}_{kh}\left(\cdot\,\vert\, s_{h}^{k}, a_{h}^{k}\right) - P_{h}\left(\cdot\,\vert\, s_{h}^{k}, a_{h}^{k}\right)}, \tilde{V}_{h + 1}^{\pi_{k}} - V_{h + 1}^{*}\big\rangle\\
    &\leq \beta_{kh}\left(s_{h}^{k}, a_{h}^{k}\right) + \big\langle \abs{\hat{P}_{kh}\left(\cdot\,\vert\, s_{h}^{k}, a_{h}^{k}\right) - P_{h}\left(\cdot\,\vert\, s_{h}^{k}, a_{h}^{k}\right)}, \tilde{V}_{h + 1}^{\pi_{k}} - V_{h + 1}^{\pi_{k}}\big\rangle \tag{$V_{h}^{*}(s) \geq V_{h}^{\pi_{k}}(s)$}\\
    &= \beta_{kh}\left(s_{h}^{k}, a_{h}^{k}\right) + \big\langle\abs{\hat{P}_{kh}\left(\cdot\,\vert\, s_{h}^{k}, a_{h}^{k}\right) - P_{h}\left(\cdot\,\vert\, s_{h}^{k}, a_{h}^{k}\right)}, \tilde{\Delta}_{h + 1}^{k}\big\rangle\\
    &\leq \beta_{kh}\left(s_{h}^{k}, a_{h}^{k}\right)  + \frac{\tilde{\Delta}_{h + 1}^{k}\left(s_{h + 1}^{k}\right)}{H} + \frac{2HSL}{N_{kh}'\left(s_{h}^{k}, a_{h}^{k}\right)} + \frac{4H^2 SL}{N_{kh}'\left(s_{h}^{k}, a_{h}^{k}\right)} \\
    &+ \sqrt{\frac{4L}{N_{kh}'\left(s_{h}^{k}, a_{h}^{k}\right)}}\left[\left(\sum_{s'\in G_{kh}}P_{h}\left(s'\,\vert\,s_{h}^{k}, a_{h}^{k}\right)\frac{\tilde{\Delta}_{h + 1}^{k}\left(s'\right)}{\sqrt{P_{h}\left(s'\,\vert\,s_{h}^{k}, a_{h}^{k}\right)}}\right) - \frac{\tilde{\Delta}_{h + 1}^{k}\left(s_{h + 1}^{k}\right)}{\sqrt{P_{h}\left(s_{h + 1}^{k}\,\vert\,s_{h}^{k}, a_{h}^{k}\right) }}\right]
\end{align*}
where the final inequality follows from Lemma \ref{lemma: correction term}. Thus, utilising the indicator variable, we have that: 
\begin{align*}
     \tilde{\Delta}_{h}^{k}(s_{h}^{k})
     &\leq  \left(1 + \frac{C}{H}\right)\tilde{\Delta}_{h + 1}^{k}\left(s_{h + 1}^{k}\right) + 2\beta_{kh}\left(s_{h}^{k}, a_{h}^{k}\right) + \frac{2CHSL}{N_{kh}'\left(s_{h}^{k}, a_{h}^{k}\right)} + \frac{4CH^2 SL}{N_{kh}'\left(s_{h}^{k}, a_{h}^{k}\right)}\\
    &\quad+ \big\langle P_{h}\left(\cdot\,\vert\,s_{h}^{k}, a_{h}^{k}\right),\, \tilde{\Delta}_{h + 1}^{k}\left(\cdot\right)\big\rangle  - \tilde{\Delta}_{h + 1}^{k}\left(s_{h + 1}^{k}\right)\\
    &\quad+ \sqrt{\frac{4CL}{N_{kh}'\left(s_{h}^{k}, a_{h}^{k}\right)}}\left[\left(\sum_{s'\in G_{kh}}P_{h}\left(s'\,\vert\,s_{h}^{k}, a_{h}^{k}\right)\frac{\tilde{\Delta}_{h + 1}^{k}\left(s'\right)}{\sqrt{P_{h}\left(s'\,\vert\,s_{h}^{k}, a_{h}^{k}\right)}}\right) - \frac{\tilde{\Delta}_{h + 1}^{k}\left(s_{h + 1}^{k}\right)}{\sqrt{P_{h}\left(s_{h + 1}^{k}\,\vert\,s_{h}^{k}, a_{h}^{k}\right) }}\right]\\
     &\leq  \left(1 + \frac{C}{H}\right)\tilde{\Delta}_{h + 1}^{k}\left(s_{h + 1}^{k}\right) + 2\beta_{kh}\left(s_{h}^{k}, a_{h}^{k}\right) + \frac{6CH^2 SL}{N_{kh}'\left(s_{h}^{k}, a_{h}^{k}\right)}\\
    &\quad+ \big\langle P_{h}\left(\cdot\,\vert\,s_{h}^{k}, a_{h}^{k}\right),\, \tilde{\Delta}_{h + 1}^{k}\left(\cdot\right)\big\rangle  - \tilde{\Delta}_{h + 1}^{k}\left(s_{h + 1}^{k}\right)\\
    &\quad+ \sqrt{\frac{4CL}{N_{kh}'\left(s_{h}^{k}, a_{h}^{k}\right)}}\left[\left(\sum_{s'\in G_{kh}}P_{h}\left(s'\,\vert\,s_{h}^{k}, a_{h}^{k}\right)\frac{\tilde{\Delta}_{h + 1}^{k}\left(s'\right)}{\sqrt{P_{h}\left(s'\,\vert\,s_{h}^{k}, a_{h}^{k}\right)}}\right) - \frac{\tilde{\Delta}_{h + 1}^{k}\left(s_{h + 1}^{k}\right)}{\sqrt{P_{h}\left(s_{h + 1}^{k}\,\vert\,s_{h}^{k}, a_{h}^{k}\right) }}\right]
\end{align*}
as required. 
\end{proof}

\begin{lemma}\label{lemma: correction term}
Let $\gamma_{kh}(s_{h}^{k}, a_{h}^{k}) \coloneqq \langle \hat{P}_{kh}\left(\cdot\,\vert\, s_{h}^{k}, a_{h}^{k}\right) - P_{h}\left(\cdot\,\vert\, s_{h}^{k}, a_{h}^{k}\right), \tilde{\Delta}_{h + 1}^{k}\rangle$. Then, with probability at least $1 - \delta'$:
\begin{align*}
    \gamma_{kh}\left(s_{h}^{k}, a_{h}^{k}\right) &\leq \frac{\tilde{\Delta}_{h + 1}^{k}\left(s_{h + 1}^{k}\right)}{H} + \frac{2HSL}{N_{kh}'\left(s_{h}^{k}, a_{h}^{k}\right)} + \frac{4H^2 SL}{N_{kh}'\left(s_{h}^{k}, a_{h}^{k}\right)} \\
    &+ \sqrt{\frac{4L}{N_{kh}'\left(s_{h}^{k}, a_{h}^{k}\right)}}\left[\left(\sum_{s'\in G_{kh}}P_{h}\left(s'\,\vert\,s_{h}^{k}, a_{h}^{k}\right)\frac{\tilde{\Delta}_{h + 1}^{k}\left(s'\right)}{\sqrt{P_{h}\left(s'\,\vert\,s_{h}^{k}, a_{h}^{k}\right)}}\right) - \frac{\tilde{\Delta}_{h + 1}^{k}\left(s_{h + 1}^{k}\right)}{\sqrt{P_{h}\left(s_{h + 1}^{k}\,\vert\,s_{h}^{k}, a_{h}^{k}\right) }}\right]
\end{align*}
where $L = \log(S^2 AH \pi^2 / 6\delta')$ and
\begin{align*}
G_{kh} &\coloneqq \{s': P_{h}\left(s'\,\vert\,s_{h}^{k}, a_{h}^{k}\right) N_{kh}'\left(s_{h}^{k}, a_{h}^{k}\right) \geq 4H^2 L\}
\end{align*}
for all $\mathcal{S}\times\mathcal{A}\times\mathcal{H}$ and $K \in \mathbb{N}_{1}$.
\end{lemma}
\begin{proof}
For completeness, we present proof of this claim and note that the ideas found here were first introduced by \citet{UCBVI}.\\

We upper bound the so-called "correction term", $C\langle \hat{P}_{kh}\left(\cdot\,\vert\, s_{h}^{k}, a_{h}^{k}\right) - P_{h}\left(\cdot\,\vert\, s_{h}^{k}, a_{h}^{k}\right), \tilde{\Delta}_{h + 1}^{k}\rangle$. Following \citet{UCBVI} and applying Bernstein's inequality to bound the difference between the estimated and actual transitions gives us, with probability $1 - \delta$:
\begin{align*}
    &\gamma_{kh}\left(s_{h}^{k}, a_{h}^{k}\right) =\Big\langle \left(\hat{P}_{kh} - P_{h}\right)\left(\cdot\,\vert\, s_{h}^{k}, a_{h}^{k}\right), \tilde{\Delta}_{h + 1}^{k}\Big\rangle \nonumber\\
    &\leq 2\sum_{s'}\left(\frac{L}{N_{kh}'\left(s_{h}^{k}, a_{h}^{k}\right)} + \sqrt{\frac{P_{h}\left(s'\,\vert\,s_{h}^{k}, a_{h}^{k}\right)L}{N_{kh}'\left(s_{h}^{k}, a_{h}^{k}\right)}} \right)\tilde{\Delta}_{h + 1}^{k}\left(s'\right)\nonumber \tag{Bernstein's Inequality}\\
    &\leq 2 \left(\frac{HSL}{N_{kh}'\left(s_{h}^{k}, a_{h}^{k}\right)} +  \sum_{s'}\sqrt{\frac{P_{h}\left(s'\,\vert\,s_{h}^{k}, a_{h}^{k}\right)L}{N_{kh}'\left(s_{h}^{k}, a_{h}^{k}\right)}}\tilde{\Delta}_{h + 1}^{k}\left(s'\right)\right)\nonumber\\
    &= 2\left(\frac{HSL}{N_{kh}'\left(s_{h}^{k}, a_{h}^{k}\right)} + \sum_{s'\not\in G_{kh}}\sqrt{\frac{P_{h}\left(s'\,\vert\,s_{h}^{k}, a_{h}^{k}\right)L}{N_{kh}'\left(s_{h}^{k}, a_{h}^{k}\right)}}\tilde{\Delta}_{h + 1}^{k}\left(s'\right) + \sum_{s'\in G_{kh}}\sqrt{\frac{P_{h}\left(s'\,\vert\,s_{h}^{k}, a_{h}^{k}\right)L}{N_{kh}'\left(s_{h}^{k}, a_{h}^{k}\right)}}\tilde{\Delta}_{h + 1}^{k}\left(s'\right) \right)
\end{align*}

By definition, $P_{h}(s'\vert s, a) < 4H^2 L/N_{kh}'(s, a)$ whenever $s'\not\in G_{kh}$, which follows simply from rearranging terms in the definition of $G_{kh}$. Therefore, 
\begin{align*}
    \sum_{s'\not\in G_{kh}}\sqrt{\frac{P_{h}\left(s'\,\vert\,s_{h}^{k}, a_{h}^{k}\right)L}{N_{kh}'\left(s_{h}^{k}, a_{h}^{k}\right)}}\tilde{\Delta}_{h + 1}^{k}\left(s'\right) &\leq \sum_{s'\not\in G_{kh}}\frac{2H L}{N_{kh}'\left(s_{h}^{k}, a_{h}^{k}\right)}\tilde{\Delta}_{h + 1}^{k}\left(s'\right) \leq \frac{2H^2 S L}{N_{kh}'\left(s_{h}^{k}, a_{h}^{k}\right)}
\end{align*}
Now, we focus on the $s'\in G_{kh}$. 
\begin{align*}
    &\sum_{s'\in G_{kh}}\sqrt{\frac{P_{h}\left(s'\,\vert\,s_{h}^{k}, a_{h}^{k}\right)L}{N_{kh}'\left(s_{h}^{k}, a_{h}^{k}\right)}}\tilde{\Delta}_{h + 1}^{k}\left(s'\right) \\
    &= \sqrt{\frac{L}{P_{h}\left(s_{h + 1}^{k}\,\vert\,s_{h}^{k}, a_{h}^{k}\right) N_{kh}'\left(s_{h}^{k}, a_{h}^{k}\right)}}\tilde{\Delta}_{h + 1}^{k}\left(s_{h + 1}^{k}\right) - \sqrt{\frac{L}{P_{h}\left(s_{h + 1}^{k}\,\vert\,s_{h}^{k}, a_{h}^{k}\right) N_{kh}'\left(s_{h}^{k}, a_{h}^{k}\right)}}\tilde{\Delta}_{h + 1}^{k}\left(s_{h + 1}^{k}\right) \\
    &\quad + \sum_{s'\in G_{kh}}P_{h}\left(s'\,\vert\,s_{h}^{k}, a_{h}^{k}\right)\sqrt{\frac{L}{P_{h}\left(s'\,\vert\,s_{h}^{k}, a_{h}^{k}\right)N_{kh}'\left(s_{h}^{k}, a_{h}^{k}\right)}} \\
    &\leq \frac{\tilde{\Delta}_{h + 1}^{k}\left(s_{h + 1}^{k}\right)}{2H} + \sqrt{\frac{L}{N_{kh}'\left(s_{h}^{k}, a_{h}^{k}\right)}}\left[\left(\sum_{s'\in G_{kh}}P_{h}\left(s'\,\vert\,s_{h}^{k}, a_{h}^{k}\right)\frac{\tilde{\Delta}_{h + 1}^{k}\left(s'\right)}{\sqrt{P_{h}\left(s'\,\vert\,s_{h}^{k}, a_{h}^{k}\right)}}\right) - \frac{\tilde{\Delta}_{h + 1}^{k}\left(s_{h + 1}^{k}\right)}{\sqrt{P_{h}\left(s_{h + 1}^{k}\,\vert\,s_{h}^{k}, a_{h}^{k}\right) }}\right]
\end{align*}
where the inequality follows from the fact that $s'\in G_{kh}$, implying that $P_{h}\left(s'\,\vert\,s_{h}^{k}, a_{h}^{k}\right) N_{kh}'\left(s_{h}^{k}, a_{h}^{k}\right) \geq 4H^2 L$. Substituting both of the above into the initial upper bound on $\gamma_{kh}\left(s_{h}^{k}, a_{h}^{k}\right)$ gives: 
\begin{align*}
    \gamma_{kh}\left(s_{h}^{k}, a_{h}^{k}\right) &\leq \frac{\tilde{\Delta}_{h + 1}^{k}\left(s_{h + 1}^{k}\right)}{H} + \frac{2HSL}{N_{kh}'}\left(s_{h}^{k}, a_{h}^{k}\right) + \frac{4H^2 SL}{N_{kh}'}\left(s_{h}^{k}, a_{h}^{k}\right) \\
    &+ \sqrt{\frac{4L}{N_{kh}'\left(s_{h}^{k}, a_{h}^{k}\right)}}\left[\left(\sum_{s'\in G_{kh}}P_{h}\left(s'\,\vert\,s_{h}^{k}, a_{h}^{k}\right)\frac{\tilde{\Delta}_{h + 1}^{k}\left(s'\right)}{\sqrt{P_{h}\left(s'\,\vert\,s_{h}^{k}, a_{h}^{k}\right)}}\right) - \frac{\tilde{\Delta}_{h + 1}^{k}\left(s_{h + 1}^{k}\right)}{\sqrt{P_{h}\left(s_{h + 1}^{k}\,\vert\,s_{h}^{k}, a_{h}^{k}\right) }}\right]
\end{align*}
completing the proof.
\end{proof}

\subsection{Missing Theoretical Results for Delayed Rewards}\label{sec: missing proof for delayed rewards}

In this section, we describe how to use active or lazy updating in the setting where only the rewards return in delay. We assume the rewards are stochastic and their expected values are unknown.\\

In the setting of delayed rewards, the agent returns the state-action pairs $\{s_h^k, a_h^k\}_{h=1}^H$ at the end of episode $k$, immediately. However, the rewards $\{r_h^k\}_{h=1}^H$ return with a random delay $\tau_k$. Since it is only the rewards that return in delay, we can estimate the transitions at the start of each episode, as usual. Thus, we apply active or lazy updating to the estimation of the expected reward function only. \\

For active updating, this amounts to estimating the expected reward function as soon as new feedback arrives: 
\begin{equation*}
    \hat{r}_{kh}\left(s, a\right) = \frac{1}{N_{kh}'\left(s, a\right)} \sum_{i = 1}^{k - 1} r_{h}^{i}\mathds{1}\{s_{h}^{i} = s, a_{h}^{i} = a, i + \tau_{i} < k\}
\end{equation*}

For lazy updating, this amounts to waiting until the observed number of rewards for a state-action-step triple have doubled before starting a new epoch. When estimating the expected reward function for $j^\text{th}$ epoch, the base algorithm will use all the available rewards: 
\begin{equation*}
    \hat{r}_{k_{j} h}\left(s, a\right) = \frac{1}{N_{k_{j} h}'\left(s, a\right)} \sum_{i = 1}^{k_{j} - 1} r_{h}^{i}\mathds{1}\{s_{h}^{i} = s, a_{h}^{i} = a, i + \tau_{i} < k\}
\end{equation*}
Using Hoeffding's inequality, one can construct confidence sets around the above estimators and derive another estimator that is optimistic, with high probability. We derive the width of the confidence set in the proof below. 
\begin{theorem}
Let $\mathfrak{R}_{K}^{P}$ denote the regret of UCRL2 from estimating the transition densities under immediate feedback. Then, with probability $1 - \delta$, the regret of UCRL2 under delayed reward is: 
\begin{equation*}
    \mathfrak{R}_{K} \lesssim \mathfrak{R}_{K}^{P} + HSA\psi_{K}^{\tau} 
\end{equation*}
for active updating. 
\end{theorem}
\begin{proof}
First, since the rewards are stochastic and their expected values are unknown, we must derive an estimator. Naturally, we use only the observed information to compute the expected value, as it is an unbiased estimator: 
\begin{equation*}
    \hat{r}_{kh}\left(s, a\right) = \frac{1}{N_{kh}'\left(s, a\right)} \sum_{i = 1}^{k - 1} r_{h}^{i}\mathds{1}\{s_{h}^{i} = s, a_{h}^{i} = a, i + \tau_{i} < k\}
\end{equation*}

Now, assume that the rewards are bounded in $[0, 1]$. Using Hoeffding's inequality, we can define an additional failure event to account for the fact that we are estimating the expected reward function: 

\begin{equation*}
    F_{k}^{r} = \left\{\exists\, s, a, h: \lvert \hat{r}_{kh}\left(s, a\right) - r_{h}\left(s, a\right)\rvert \geq \sqrt{\frac{6\log\left(2SAT\pi/6\delta'\right)}{N_{kh}'\left(s, a\right)}} := \epsilon_{kh}^{r}\left(s, a\right) \right\}
\end{equation*}

which holds across all episodes with probability $1 - \delta'$. Recall, we have a failure event for the transitions that holds with probability $1 - \delta'$ too. Thus, we get the following optimistic estimator of the expected reward function: 
\begin{equation*}
    \tilde{r}_{kh}\left(s, a\right) = \min\left\{1, \hat{r}_{kh}\left(s, a\right) +  \sqrt{\frac{6\log\left(2SAT\pi/6\delta'\right)}{N_{kh}'\left(s, a\right)}}\right\}
\end{equation*}
which upper bounds the true expected reward function with probability $1 - \delta'$ across all episodes. As in the immediate feedback setting, the failure event for the transition densities is: 
\begin{align*}
     F_{k}^{p} = \left\{\exists\, s, a, h: \norm{\hat{P}_{kh}\left(\cdot \vert s, a\right) - P_{h}\left(\cdot \vert s, a\right)}_{1} \geq \sqrt{\frac{6S \log\left(AT\pi/6\delta'\right)}{N_{kh}\left(s, a\right)}}:= \epsilon_{kh}^{p}\left(s, a\right) \right\}   
\end{align*}
where
\begin{equation*}
    \tilde{P}_{kh}\left(\cdot\vert s, a\right) \in \left\{Q \in \Delta: \norm{Q - P_{h}\left(\cdot\vert s, a\right)} \leq \epsilon_{kh}^{p}\left(s, a\right) \right\}
\end{equation*}

By optimism, and due to UCRL2 having $C = 1$: with probability $1 - 2\delta'$: 
\begin{align}
    \mathfrak{R}_{K} &=\sum_{k = 1}^{K}\Delta_{1}^{k}= \sum_{k = 1}^{K}V^{*}_{1}\left(s_{1}^{k}\right) - V^{\pi_{k}}_{1}\left(s_{1}^{k}\right) \nonumber\\
    &\leq \sum_{k = 1}^{K}\tilde{\Delta}_{1}^{k}= \sum_{k = 1}^{K}\tilde{V}^{\pi_{k}}_{1}\left(s_{1}^{k}\right) - V^{\pi_{k}}_{1}\left(s_{1}^{k}\right)\label{eqn: delayed reward regret decomp}\\
    &\leq \sum_{k = 1}^{K} \sum_{h = 1}^{H} 2H\epsilon_{kh}^{p}\left(s_{h}^{k}, a_{h}^{k}\right) + 2\epsilon_{kh}^{r}\left(s_{h}^{k}, a_{h}^{k}\right) + \zeta_{h}^{k}\left(s_{h}^{k}, a_{h}^{k}\right)\nonumber\\
    &\leq 2H\sqrt{T\log\left(\frac{K\pi}{6\delta'}\right)} + \sum_{k = 1}^{K} \sum_{h = 1}^{H} 2H\epsilon_{kh}^{p}\left(s_{h}^{k}, a_{h}^{k}\right) + \sum_{k = 1}^{K} \sum_{h = 1}^{H} 2\epsilon_{kh}^{r}\left(s_{h}^{k}, a_{h}^{k}\right)\nonumber\\
    &\leq 2H\sqrt{T\log\left(\frac{K\pi}{6\delta'}\right)} + 2H\sqrt{6S\log\left(AT\pi/6\delta'\right)}\sum_{k = 1}^{K}\sum_{h = 1}^{H}\frac{1}{\sqrt{N_{kh}\left(s_h^k, a_h^k\right)}}+ \sum_{k = 1}^{K} \sum_{h = 1}^{H} 2\epsilon_{kh}^{r}\left(s_{h}^{k}, a_{h}^{k}\right)\nonumber\\
    &\leq  2H\sqrt{T\log\left(\frac{K\pi}{6\delta'}\right)} + 4H\sqrt{6S HSAT\log\left(AT\pi/6\delta'\right)} + \sum_{k = 1}^{K} \sum_{h = 1}^{H} 2\epsilon_{kh}^{r}\left(s_{h}^{k}, a_{h}^{k}\right)\nonumber\\
    &\leq 2H\sqrt{T\log\left(\frac{K\pi}{6\delta'}\right)} + 10H^{3/2}S\sqrt{AT\log\left(AT\pi/6\delta'\right)} + \sum_{k = 1}^{K} \sum_{h = 1}^{H} 2\epsilon_{kh}^{r}\left(s_{h}^{k}, a_{h}^{k}\right)\nonumber\\
    &\leq \mathfrak{R}_{K}^{P} + \sum_{k = 1}^{K} \sum_{h = 1}^{H} 2\epsilon_{kh}^{r}\left(s_{h}^{k}, a_{h}^{k}\right)\label{eqn: delayed-reward active regret}
\end{align}

The penultimate inequality follows from Lemma \ref{lemma: p-series}. Further, 
$$\mathfrak{R}_{K}^{P} = 2H\sqrt{T\log\left(\frac{K\pi}{6\delta'}\right)} + 10H^{3/2}S\sqrt{AT\log\left(AT\pi/6\delta'\right)}$$
is the regret of the base algorithm (UCRL2) in an immediate feedback environment with know reward functions. Now, to prove the statements of the corollary, we must bound the summation of the estimation error for the rewards. Doing so is just a matter of applying Lemma \ref{lemma: summation bound}: 
\begin{align*}
    \sum_{k = 1}^{K} \sum_{h = 1}^{H} 2\epsilon_{kh}^{r}\left(s_{h}^{k}, a_{h}^{k}\right) 
    &= 2\sum_{k = 1}^{K} \sum_{h = 1}^{H} \sqrt{\frac{6\log\left(2SAT\pi/6\delta'\right)}{N_{kh}'\left(s, a\right)}} \\
    &\leq 8\sqrt{6HSAT\log\left(SAT\pi/6\delta'\right)} + 6HSA\psi_{K}^{\tau}\sqrt{6\log\left(SAT\pi/6\delta'\right)}
\end{align*}
Substituting the above into Equation \eqref{eqn: delayed-reward active regret} and omitting poly-logarithmic factors gives the stated result. 
\end{proof}

\section{Additional Experimental Results}\label{sec: additional experiments}
Here, we present additional experimental results for the chain environments with $H = S \in \{5, 10, 20\}$ and $\mathbb{E}[\tau] \in \{100, 300, 500\}$. In all combinations of chain length and expected delay, our updating procedures give better empirical performance, especially for the delay distributions with higher variances. For all expected delays, active updating gives the best performance. However, our experiments indicate that lazy updating with $\alpha = 10$ or $100$ is comparable, as one would expect based on the intuition that it is an approximation to active updating that converges in the limit as $\alpha \rightarrow\infty$.  
\subsection{Chain Environment with $H = S = 5$}
\begin{figure}[h!]
    \centering
    \includegraphics[width = \textwidth]{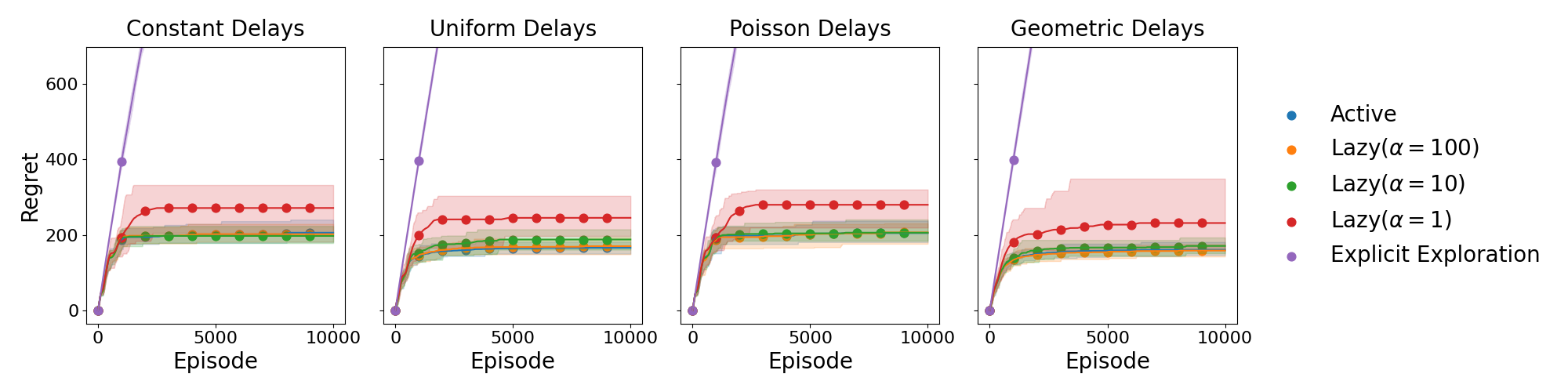}
    \caption{Cumulative Regret $\left(S = 5, \mathbb{E}[\tau] = 100\right)$.}
\end{figure}

\begin{figure}[h!]
    \centering
    \includegraphics[width = \textwidth]{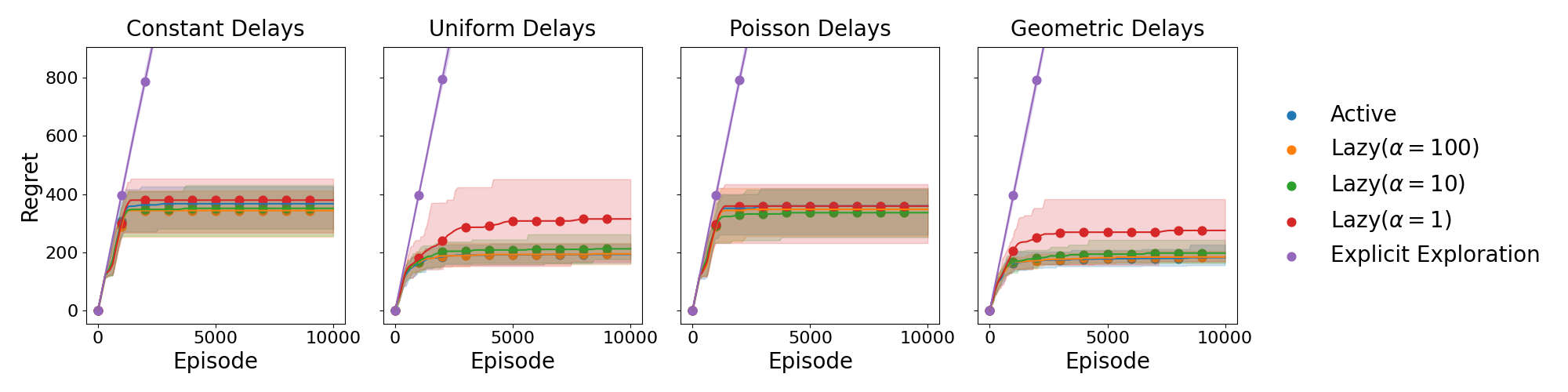}
    \caption{Cumulative Regret $\left(S = 5, \mathbb{E}[\tau] = 300\right)$.}
\end{figure}

\begin{figure}[h!]
    \centering
    \includegraphics[width = \textwidth]{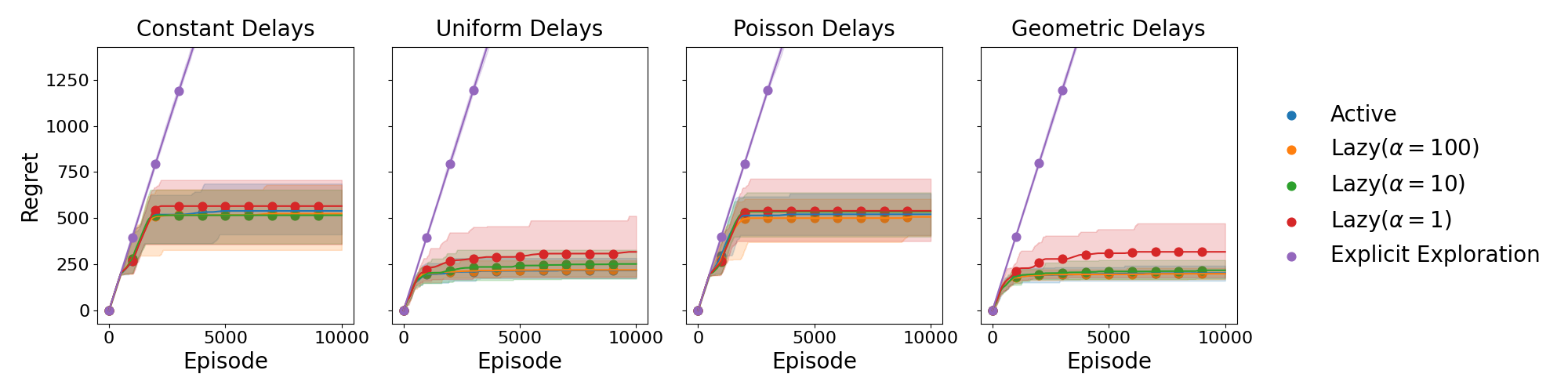}
    \caption{Cumulative Regret $\left(S = 5, \mathbb{E}[\tau] = 500\right)$.}
\end{figure}

\begin{figure}[h!]
    \centering
    \includegraphics[width = \textwidth]{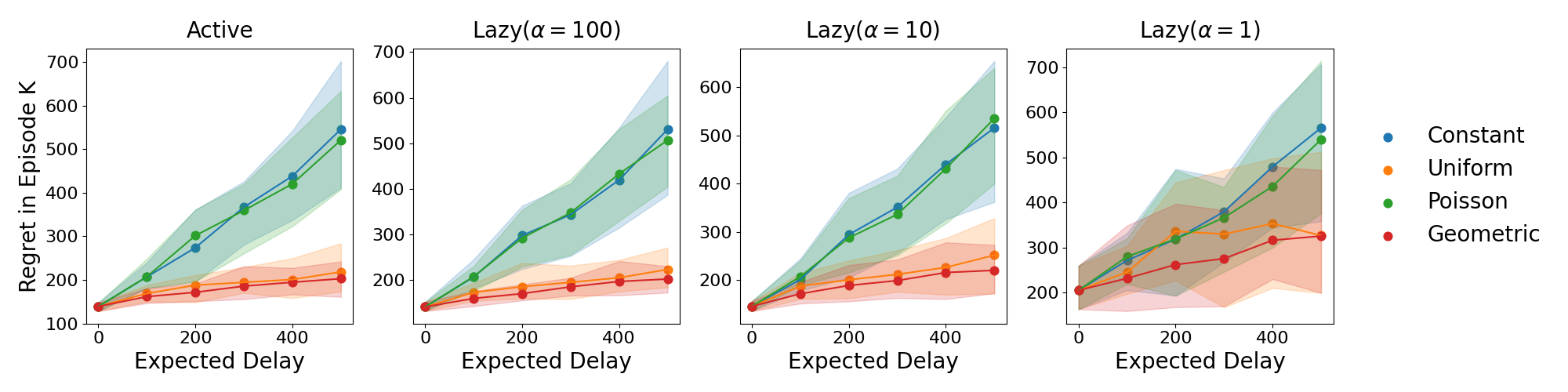}
    \caption{Delay Dependence $\left(S = 5\right)$}
\end{figure}

\newpage
\subsection{Chain Environment with $H = S = 10$}
\begin{figure*}[h!]
    \centering
    \includegraphics[width = \textwidth]{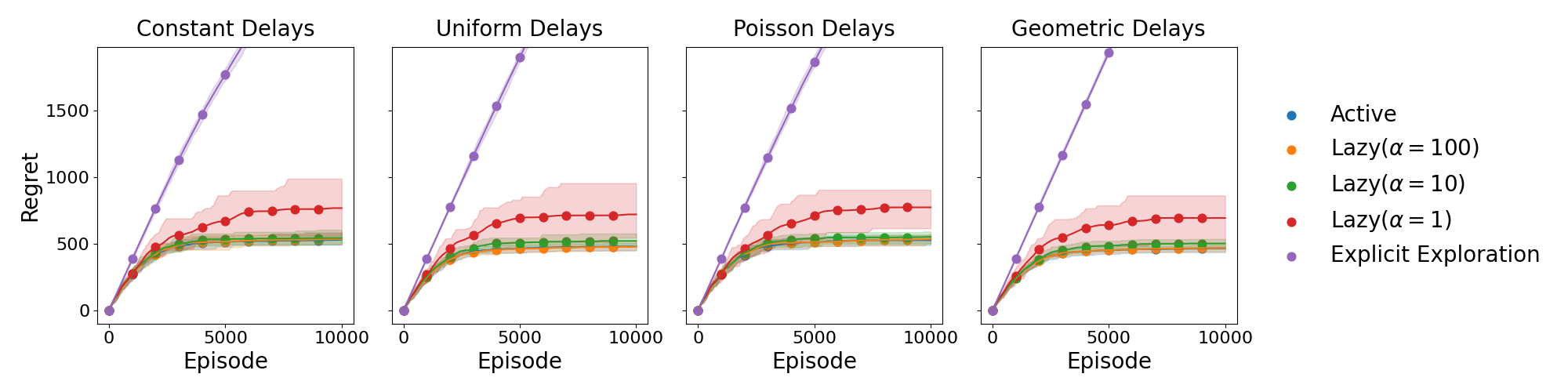}
    \caption{Cumulative Regret $\left(S = 10, \mathbb{E}[\tau] = 100\right)$.}
\end{figure*}

\begin{figure*}[h!]
    \centering
    \includegraphics[width = \textwidth]{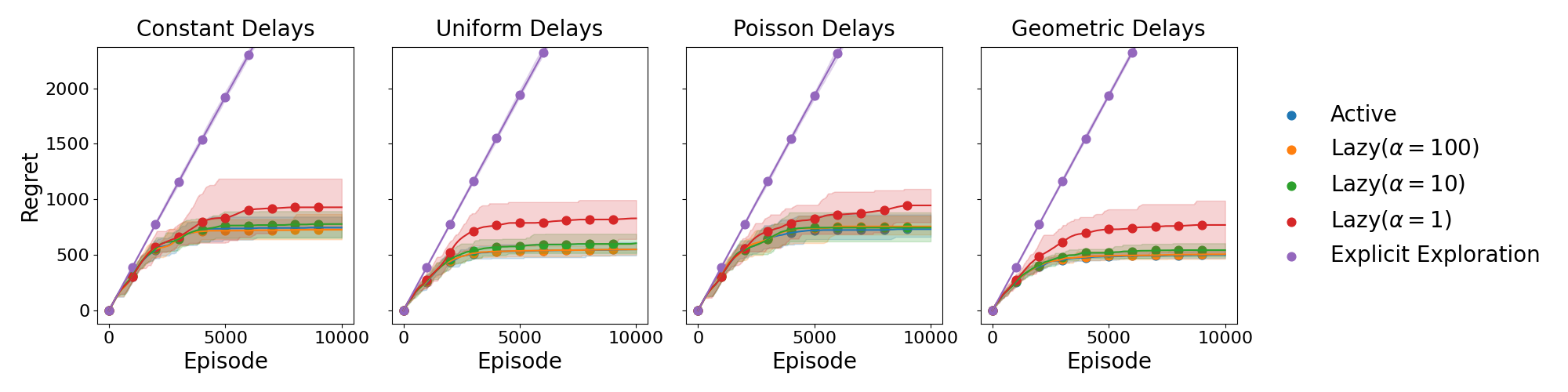}
    \caption{Cumulative Regret $\left(S = 10, \mathbb{E}[\tau] = 300\right)$.}
\end{figure*}

\begin{figure*}[h!]
    \centering
    \includegraphics[width = \textwidth]{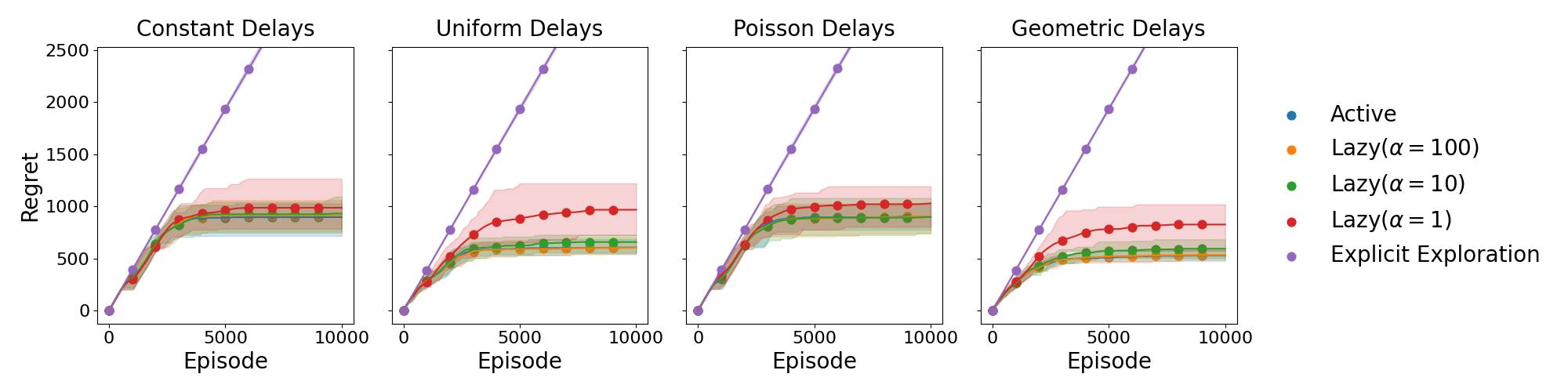}
    \caption{Cumulative Regret $\left(S = 10, \mathbb{E}[\tau] = 500\right)$.}
\end{figure*}

\begin{figure*}[h!]
    \centering
    \includegraphics[width = \textwidth]{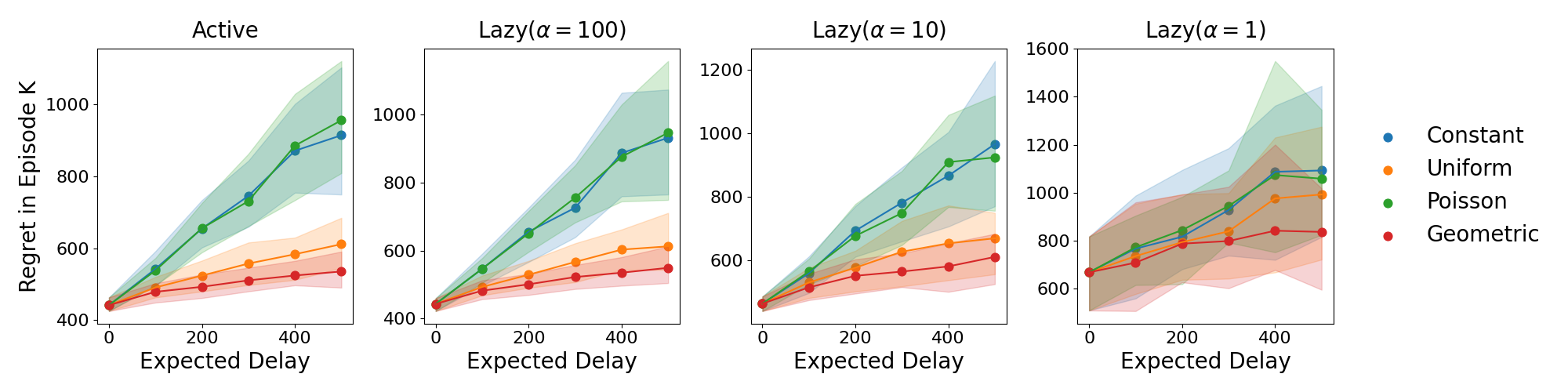}
    \caption{Delay Dependence $\left(S = 10\right)$}
\end{figure*}

\newpage
\subsection{Chain Environment with $H = S = 20$}
\begin{figure*}[h!]
    \centering
    \includegraphics[width = \textwidth]{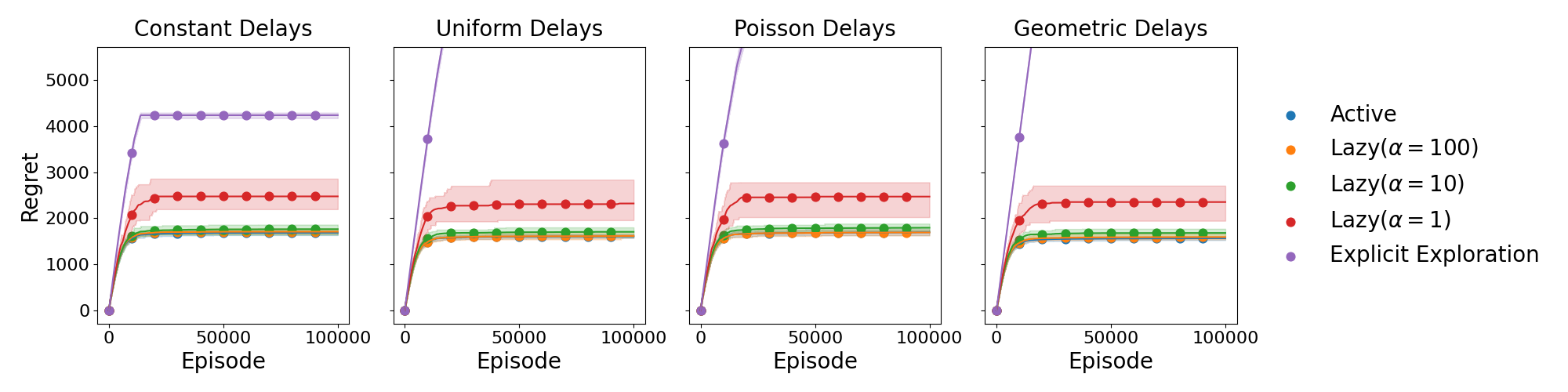}
    \caption{Cumulative Regret $\left(S = 20, \mathbb{E}[\tau] = 100\right)$.}
\end{figure*}

\begin{figure*}[h!]
    \centering
    \includegraphics[width = \textwidth]{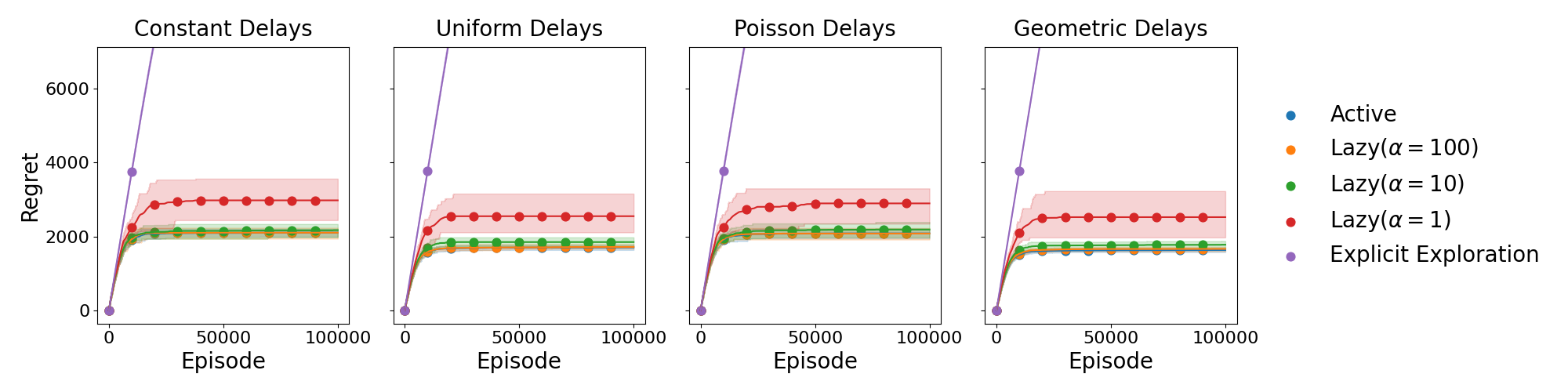}
    \caption{Cumulative Regret $\left(S = 20, \mathbb{E}[\tau] = 300\right)$.}
\end{figure*}

\begin{figure*}[h!]
    \centering
    \includegraphics[width = \textwidth]{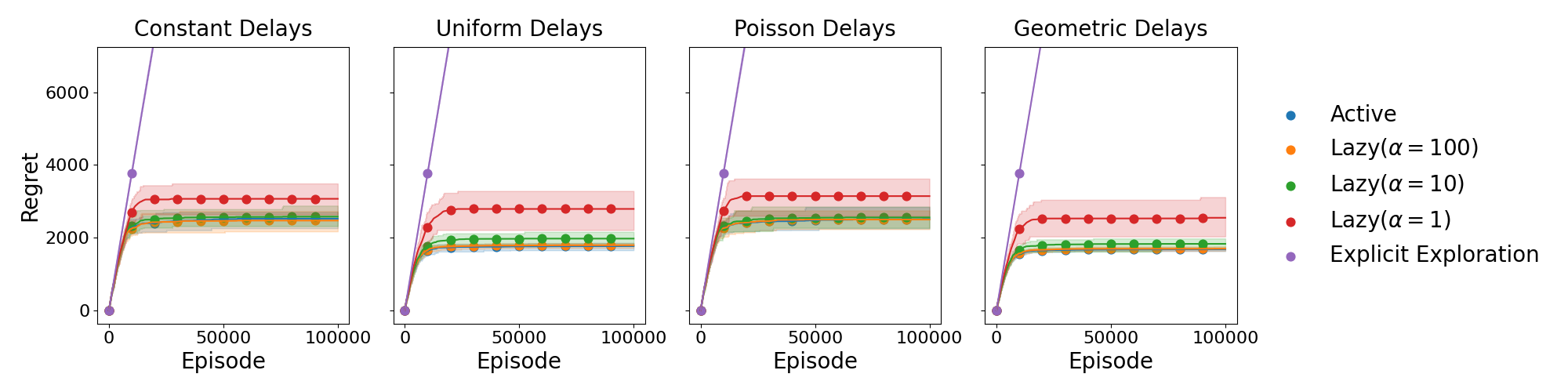}
    \caption{Cumulative Regret $\left(S = 20, \mathbb{E}[\tau] = 500\right)$.}
\end{figure*}

\begin{figure*}[h!]
    \centering
    \includegraphics[width = \textwidth]{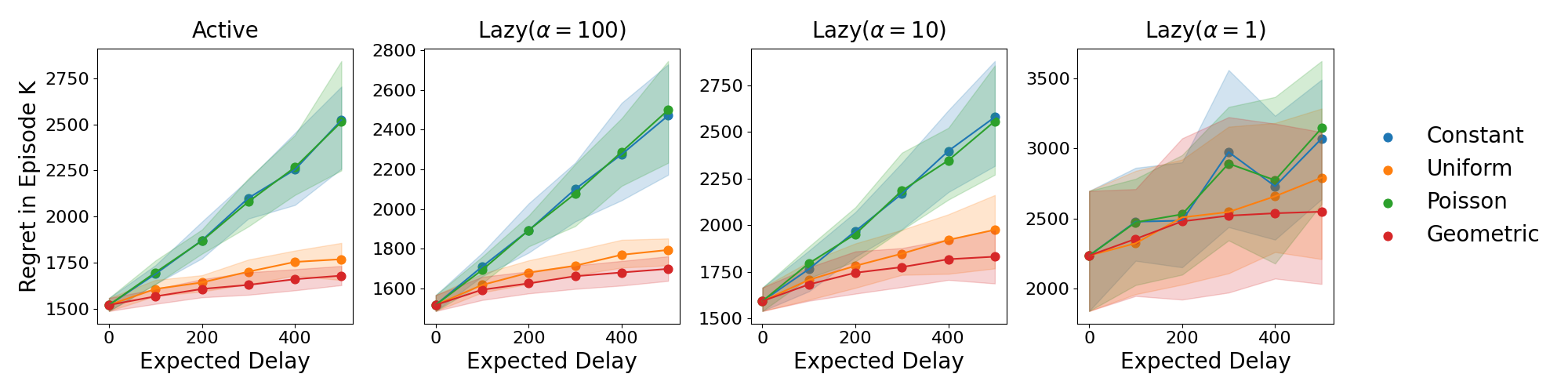}
    \caption{Delay Dependence $\left(S = 20\right)$}
\end{figure*}

\newpage
\subsection{Chain Environment with $H = S = 30$}

\begin{figure*}[h!]
    \centering
    \includegraphics[width = \textwidth]{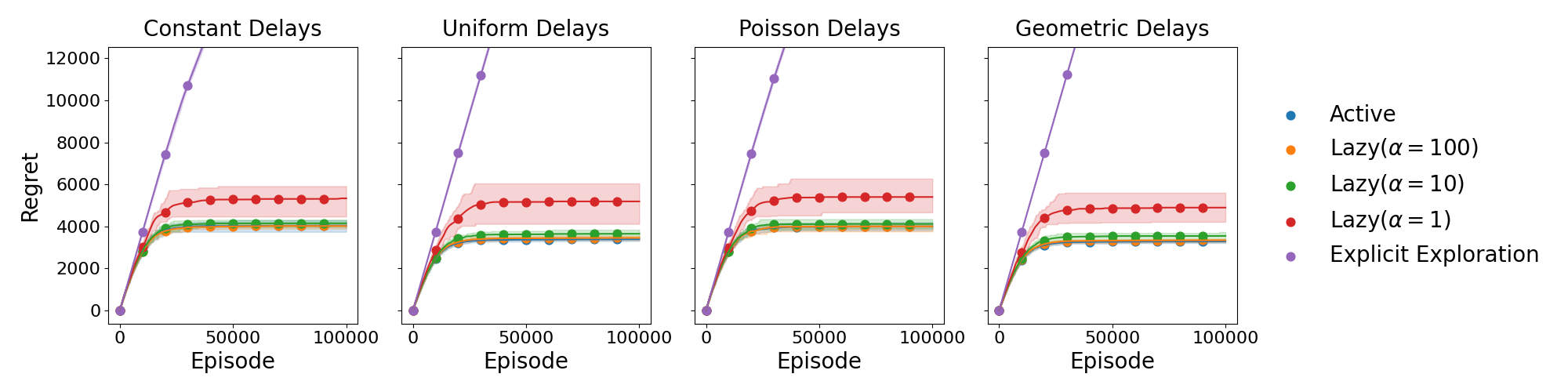}
    \caption{Cumulative Regret $\left(S = 30, \mathbb{E}[\tau] = 300\right)$.}
\end{figure*}

\begin{figure*}[h!]
    \centering
    \includegraphics[width = \textwidth]{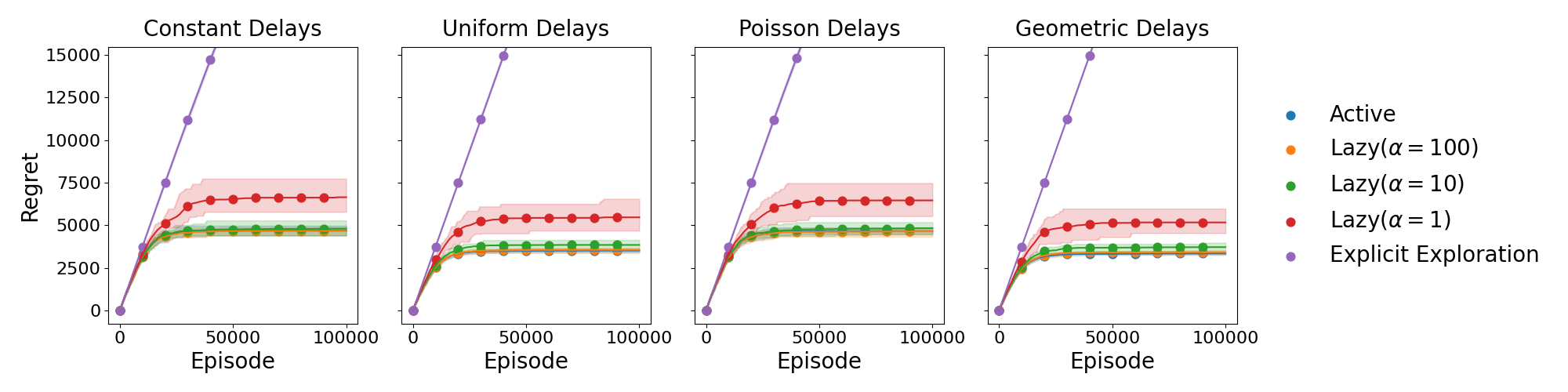}
    \caption{Cumulative Regret $\left(S = 30, \mathbb{E}[\tau] = 500\right)$.}
\end{figure*}
\end{appendix}
\end{document}